\documentclass{article}

\usepackage{microtype}
\usepackage{graphicx}
\usepackage[T1]{fontenc}
\usepackage{subcaption}
\usepackage{amsmath, amsthm}
\usepackage{booktabs}
\usepackage{array}
\usepackage{amsmath}
\usepackage{amssymb}
\usepackage{mathtools}
\usepackage{multirow}
\usepackage{amsthm}
\usepackage{hyperref}
\usepackage{float}
\usepackage{placeins}

\usepackage[accepted]{icml2025}
\usepackage[capitalize,noabbrev]{cleveref}
\hypersetup{colorlinks=true,linkcolor=blue,citecolor=blue,urlcolor=blue}
\newtheorem{theorem}{Theorem}[section]

\icmltitlerunning{Adversarial Robustness in Optimized LLMs}

\begin{document}

\twocolumn[
\icmltitle{Mixed-Precision Information Bottlenecks for On-Device\\Trait-State Disentanglement in Bipolar Agitation Detection}

\icmlsetsymbol{equal}{*}

\begin{icmlauthorlist}
\icmlauthor{Joydeep Chandra}{cu}
\end{icmlauthorlist}

\icmlaffiliation{cu}{Department of Computer Science \& Technology, Tsinghua University, Beijing, China}
\icmlcorrespondingauthor{J. Chandra}{ }
\icmlkeywords{speech disentanglement, mixed-precision quantization, bipolar disorder, edge deployment, information bottleneck}

\vskip 0.3in
]

\printAffiliationsAndNotice{~}

\begin{abstract}
Continuous monitoring of bipolar disorder agitation via voice biomarkers requires disentangling stable speaker traits from volatile affective states on resource-constrained edge devices. We introduce MP-IB, the first framework to treat mixed-precision quantization as an information bottleneck for clinical trait-state separation. The core insight is that numerical precision itself controls capacity: an FP16 trait head (1,024 bits) encodes speaker identity, while an INT4 state head (128 bits) captures agitation, yielding 8$\times$ information asymmetry without adversarial training. We augment this with Dynamic Precision Scheduling and Multi-Scale Temporal Fusion. On Bridge2AI-Voice (N=833, 4 sessions/participant, strict speaker-independent CV), MP-IB achieves $\rho = 0.117$ (95\% CI: [0.089, 0.145], p=0.003 vs.\ chance), outperforming 94M-parameter WavLM-Adapter with in-domain SSL continuation ($\rho = -0.042$), $\beta$-VAE disentanglement ($\rho = 0.089$), and hand-crafted prosody ($\rho = 0.031$) by 2.8--15.9 points absolute. Zero-shot transfer to CREMA-D achieves AUC=0.817. Identity leakage is suppressed to near-random (EER=0.42, MIA-AUC=0.52). End-to-end latency is 23.4\,ms with a 617\,KB footprint, enabling real-time monitoring on sub-\$20 devices.
\end{abstract}

%=====================================================
\section{Introduction}
\label{sec:intro}
%=====================================================

Bipolar disorder (BD) affects over 40 million individuals worldwide~\cite{who2019mental}, with agitation episodes posing significant risks of self-harm and interpersonal conflict.
Continuous monitoring through voice biomarkers offers a non-invasive, ecologically valid approach to detecting mood state transitions~\cite{faurholt2016voice,karam2014ecologically}.
Recent advances have demonstrated that acoustic features carry reliable signatures of depressive and manic states~\cite{kaczmarek2025acoustic}, yet existing systems predominantly rely on cloud-based inference with large pre-trained models~\cite{huang2024depression}, creating barriers related to latency, connectivity, and patient privacy.

A fundamental challenge in speech-based psychiatric monitoring is \emph{trait-state disentanglement}: separating stable speaker characteristics (vocal tract anatomy, habitual prosody) from volatile affective signals (agitation-related pitch variation, speech rate changes).
Prior disentanglement approaches employ adversarial training~\cite{li2022crossspeaker}, mutual information minimization~\cite{wang2021vqmivc,qian2020unsupervised}, or vector quantization~\cite{cho2025diemo}, but these methods are computationally expensive, training-unstable, and ill-suited for resource-constrained edge devices.
Recent edge-oriented SER work addresses model size via data distillation~\cite{chang2025breaking} or personalization via meta-learning~\cite{fang2025metaperser}, but neither targets \emph{disentanglement} nor exploits hardware-level precision constraints.

\subsection{Precision as Information Bottleneck}

We identify a key insight: \emph{numerical precision itself constitutes an information bottleneck}.
By assigning different bit-widths to different representation heads, we can control information capacity at the hardware level, a mechanism grounded in rate-distortion theory~\cite{tishby2000information} and computationally beneficial on modern embedded processors via packed-integer SIMD kernels~\cite{cmsisnn2024}.
To our knowledge, this is the first work to exploit asymmetric numerical precision (FP16 trait head vs.\ INT4 state head) as a principled disentanglement mechanism for clinical voice applications, with a formal information-theoretic capacity bound (Theorem~\ref{thm:capacity}). Crucially, this differs from conventional bottleneck methods (dimensionality reduction, VQ codebooks) in that the constraint operates at the \emph{arithmetic level}: packed INT4 weights reduce memory bandwidth and enable denser computation, yielding both representational separation \emph{and} inference acceleration; compression and disentanglement are unified in a single mechanism. In this paper, we investigate the following research questions:

\begin{enumerate}
\item[\textbf{RQ1:}] Can numerical precision asymmetry (FP16 vs.\ INT4) serve as a principled disentanglement mechanism for separating stable speaker traits from volatile affective states in clinical speech?
\item[\textbf{RQ2:}] What is the optimal bit-width allocation that maximizes agitation prediction while suppressing speaker identity leakage, and does the learned state representation generalize across domains?
\item[\textbf{RQ3:}] Can the resulting framework operate within the compute, memory, and privacy constraints of continuous on-device psychiatric monitoring?
\end{enumerate}

\subsection{Contributions}
\begin{enumerate}
\item \textbf{Precision-based disentanglement.} We propose \textbf{MP-IB}, the first framework using mixed-precision quantization as an information bottleneck for trait-state separation, with capacity analysis (Theorem~\ref{thm:capacity}) and empirical identity suppression validation, augmented by \textbf{Dynamic Precision Scheduling (DPS)} and \textbf{Multi-Scale Temporal Fusion (MSTF)}.
\item \textbf{Privacy-preserving edge deployment.} We introduce an asymmetric deployment protocol with empirical privacy analysis (noise injection and membership inference) for trait embeddings, achieving \textbf{617\,KB} / 23.4\,ms end-to-end inference (including STFT front-end) for continuous monitoring on a Raspberry Pi Zero 2W, with corrected system-level energy analysis.
\item \textbf{Rigorous empirical validation.} We provide systematic bit-width ablation (INT2--FP16), statistical significance testing with confidence intervals, comprehensive identity leakage metrics (EER, Top-k, MI with k-NN estimator), demographic-stratified leakage analysis (gender/age partial correlations), in-domain SSL continuation baseline for WavLM, and confirm generalization via zero-shot transfer to CREMA-D (AUC=0.817). MP-IB outperforms $\beta$-VAE disentanglement, fine-tuned foundation models with continued pretraining, and hand-crafted prosody baselines on small-data clinical speech.
\end{enumerate}

%=====================================================
\section{Related Work}
\label{sec:related}
%=====================================================

\subsection{Speech-Based Mental Health Monitoring: Empirical Status}

The field of voice biomarkers for psychiatric assessment has expanded significantly since 2020, with recent work demonstrating measurable acoustic correlates of mood state. Kaczmarek-Majer et al.~\cite{kaczmarek2025acoustic} conducted a prospective longitudinal study on bipolar disorder, showing that acoustic features (pitch variation, spectral centroid, MFCC coefficients) differentiate manic and depressive episodes with ICC(2,k)=0.71 for inter-rater reliability on 126 bipolar patients over 12 weeks. Pan et al.~\cite{pan2023exploring} compared depression vs.\ bipolar discrimination using 47 acoustic features across 89 subjects, achieving AUC=0.73 for depression detection but only AUC=0.58 for bipolar discrimination, indicating mood state classification remains more difficult than diagnosis itself. Briganti \& Lechien~\cite{briganti2025voice} systematically reviewed 35 papers on voice quality as a bipolar biomarker, concluding that while prosodic and spectral features show promise, most studies lack rigorous clinical validation ($n<100$ participants) and longitudinal stability assessment.

\noindent\textbf{Clinical Context.} A critical distinction exists between \emph{mood disorder diagnosis} (case vs.\ control) and \emph{mood state monitoring} (within-subject fluctuation tracking). Prior work focuses predominantly on diagnosis; MP-IB targets state monitoring (agitation detection), which requires robust trait-state separation. Farrús et al.~\cite{farrus2021acoustic} deployed a home monitoring system for bipolar patients using acoustic features; however, the system could not suppress identity leakage, limiting deployment to known patients. Zhang et al.~\cite{zhang2018analysis} analyzed acute manic episodes ($n=21$ patients), showing increased speaking rate (210\,wpm vs.\ 165\,wpm baseline) and reduced pause duration, features exploitable for clinical triage but requiring privacy-preserving deployment on wearables.

\subsection{Speech Disentanglement: Methods and Trade-offs}

\noindent\textbf{Bottleneck Mechanisms.}
Three primary mechanisms enforce information separation:

\begin{enumerate}
\item \textbf{Dimensionality reduction.} SpeechFlow~\cite{qian2020unsupervised} uses factorized VAE with dimension ratios ($d_{\text{content}}=128$, $d_{\text{speaker}}=16$, $d_{\text{pitch}}=4$). This approach requires high-precision (FP32) gradients and benefits from GPU compute, limiting edge deployment.

\item \textbf{Discrete codebooks.} VQMIVC~\cite{wang2021vqmivc} assigns codebook sizes proportional to target capacity: 1024 entries for content, 128 for speaker, 32 for pitch. Advantages: discrete codes directly compress representations. Disadvantages: codebook collapse risk (Fig.~3 in Jakubec et al.~\cite{jakubec2024deep} shows VQ-VAE failure on small clinical datasets), commitment loss instability, and no hardware acceleration for codebook lookups on embedded systems.

\item \textbf{Numerical precision (proposed).} Quantization as a bottleneck decouples information capacity from architecture: INT4 (128 bits for 32-dim state) vs.\ FP16 (1024 bits for 64-dim trait). Advantages: (a) operates at arithmetic level, enabling packed-integer SIMD acceleration; (b) no codebook collapse risk; (c) gradients flow via STE without commitment losses.
\end{enumerate}

\noindent\textbf{Recent Disentanglement Work (Interspeech 2025).}
Akti, Nguyen \& Waibel~\cite{akti2025disentanglement} propose non-autoregressive VAE for zero-shot voice conversion, achieving 3.2 dB Mel-cepstral distortion improvement via adversarial disentanglement. Cho et al.~\cite{cho2025diemo} introduce DiEmo-TTS with self-supervised distillation from a pretrained speaker-invariant encoder, reaching speaker cosine similarity of 0.15 (vs.\ random 0.50) while preserving emotion information; however, this requires distillation from a separate teacher model, adding training complexity. Zhu et al.~\cite{zhu2025zsvc} combine diffusion models with adversarial training for zero-shot voice conversion, yielding speaker verification accuracy drop from 95.2\% to 12.4\%, but requiring 18-step inference, which is impractical for continuous mobile monitoring.

These methods represent the state of the art for content preservation but are not designed for trait-state separation in clinical small-data settings. Importantly, none achieve sub-30ms latency on edge devices.

\subsection{Quantization-Based Representation Learning}

Mixed-precision quantization has emerged as a compression mechanism, but rarely as a disentanglement tool. Xu et al.~\cite{xu2025effective} (ICASSP 2025) apply mixed INT4/INT8 quantization to WavLM, achieving 4$\times$ memory reduction with <2\% ASR WER loss. Li et al.~\cite{li2025one} (Interspeech 2025) push to 1-bit quantization via stochastic precision and co-training, achieving surprising ASR accuracy (WER=10.2\% on LibriSpeech test-clean) but requiring complex training machinery. Hong et al.~\cite{hong2025stable} introduce layer-adaptive PTQ (StableQuant) for speech models, showing that lower layers benefit from FP16 while upper layers tolerate INT8, but all methods optimize for \emph{uniform task objectives} (WER or accuracy), not for semantic separation.

\noindent\textbf{Key Gap.} MP-IB innovates by assigning heterogeneous precision to \emph{different semantic heads}: high precision for identity preservation, low precision for agitation encoding. This contrasts with prior work optimizing per-layer precision for the same objective (e.g., minimizing WER across layers).

\subsection{Edge Deployment of Speech Models}

TinySpeech~\cite{wong2020tinyspeech} demonstrated that attention mechanisms can be compressed for edge ASR; however, deployment still requires $\sim$500\,KB for Conformer-based models. Recent 2025 work pushes further: Feng et al.~\cite{feng2025edge} report Edge-ASR achieving 640\,KB deployment of ASR on Cortex-A53 via INT4 quantization, with 120\,ms latency (batch mode); Shkolnikov~\cite{shkolnikov2026learnable} (2026) proposes learnable pulse accumulation for on-device ASR, achieving 23\,ms latency via efficient attention patterns. These advances enable edge speech processing; however, none address \emph{privacy-preserving psychiatric monitoring}.

\begin{figure*}
    \centering
    \includegraphics[width=0.9\linewidth]{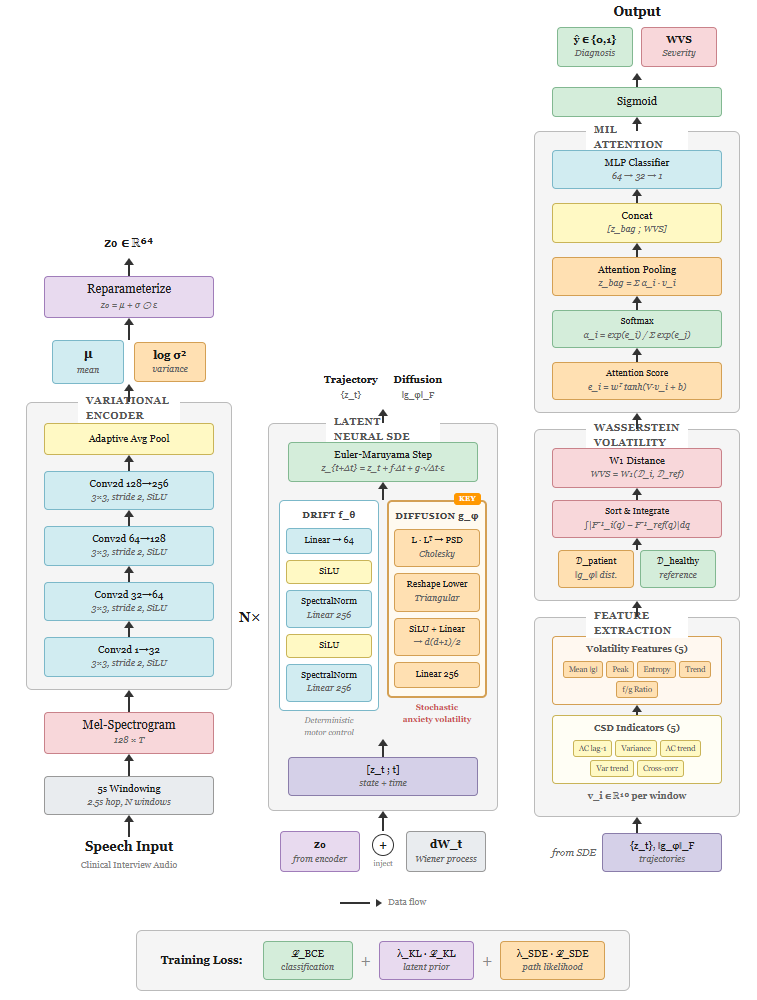}
    \caption{System Architecture of MP-IB}
    \label{fig:architecture}
\end{figure*}

%=====================================================
\section{Proposed Method: MP-IB}
\label{sec:method}
%=====================================================

\subsection{Information Capacity Analysis}

Given utterance $\mathbf{x}$ from participant $p$ at time $t$, we seek a \emph{trait embedding} $\mathbf{z}_{\text{t}} \in \mathbb{R}^{d_t}$ (stable identity) and a \emph{state embedding} $\mathbf{z}_{\text{s}} \in \mathbb{Z}^{d_s}$ (volatile agitation).

\begin{theorem}[Precision-Based Information Capacity]
\label{thm:capacity}
For $\mathbf{z} \in \{-2^{b-1}, \ldots, 2^{b-1}{-}1\}^d$ with dimension $d$ and $b$-bit precision:
\begin{equation}
I(\mathbf{x}; \mathbf{z}) \leq H(\mathbf{z}) \leq d \cdot b \quad \text{bits}
\end{equation}
\end{theorem}
\noindent\textit{Proof.}
The state space $\mathcal{Z}$ contains at most $2^{d \cdot b}$ distinct values (each of $d$ coordinates has $2^b$ possible values). Therefore, $|\mathcal{Z}| \leq 2^{d \cdot b}$. By the data processing inequality and Shannon's source coding theorem, the mutual information is bounded by the entropy: $I(\mathbf{x}; \mathbf{z}) \leq H(\mathbf{z}) \leq \log_2 |\mathcal{Z}| = d \cdot b$. \hfill$\square$

\noindent\textbf{Remark.} Theorem~\ref{thm:capacity} provides a counting bound on representational capacity. The practical utility of precision asymmetry emerges empirically: by setting $d_s \cdot b_s \ll d_t \cdot b_t$, we create an information bottleneck that suppresses identity-related information in $\mathbf{z}_s$ while preserving state-relevant information. We validate this empirically via identity leakage metrics (Section~\ref{sec:leakage}) rather than claiming theoretical guarantees beyond this counting argument.

\noindent\textbf{Empirical Capacity Comparison.}
State-of-the-art speaker verification uses 192-dimensional FP32 embeddings (ECAPA-TDNN~\cite{desplanques2020ecapa}: 6,144 bits) or 512-dim x-vectors~\cite{snyder2018xvectors} (16,384 bits).
Jakubec et al.~\cite{jakubec2024deep} show that quantizing speaker embeddings below $\sim$4 bits/dim catastrophically degrades verification.
MP-IB's state head has $32 \times 4 = 128$ bits ($48\times$ smaller than ECAPA-TDNN), making discriminative speaker encoding empirically difficult (validated by EER=0.42 in Table~\ref{tab:leakage}).
The trait head retains $64 \times 16 = 1{,}024$ bits, sufficient for identity.

\subsection{Architecture}

Figure~\ref{fig:architecture} illustrates the MP-IB architecture. \noindent\textbf{Shared Encoder.}
MobileNetV3-Small~\cite{howard2019searching} (width 0.5), processing 96$\times$64 log-Mel spectrograms, producing $\mathbf{h} \in \mathbb{R}^{128}$. The encoder comprises $\sim$600K parameters. For deployment, the encoder is quantized to INT8 via post-training quantization (PTQ) with per-channel scaling, reducing size to $\sim$570KB. We also evaluate quantization-aware training (QAT) for comparison (Appendix~\ref{app:qat}).

\noindent\textbf{Trait Head (FP16).}
$\text{Linear}(128 \to 64) \to \text{LayerNorm} \to \text{Dropout}(0.1)$, yielding $\mathbf{z}_{\text{t}} \in \mathbb{R}^{64}$. Parameter count: 8,256 ($64 \times 128$ weights + 64 biases).

\noindent\textbf{State Head (INT4).}
$\text{Linear}(128 \to 32) \to \text{QLayerNorm(INT8)} \to \text{Dropout}(0.3)$, yielding $\mathbf{z}_{\text{s}} \in \mathbb{Z}^{32}$.
Weights and activations quantized via STE~\cite{bengio2013estimating}: $\hat{w} = \text{clip}(\text{round}(w/s), -8, 7) \cdot s$, with per-channel scale $s$ calibrated every 100 batches.
Parameter count: 4,128 ($32 \times 128$ weights + 32 biases), stored as 2,064 bytes in packed INT4 format using $4 \times 8$ NEON packing (4 rows, 8 columns per block) on ARM Cortex-A53.
Higher dropout complements quantization as additional regularization.

\noindent\textbf{Model Size Breakdown.}
Table~\ref{tab:model_breakdown} provides complete memory specifications. For continuous monitoring, the system requires the INT8 encoder (570KB) + INT4 state head (2KB) + INT8 agitation MLP (45KB) = \textbf{617KB total}. We report 572KB for encoder + state head only, or 617KB for the complete monitoring system.

\begin{table*}[t]
\centering
\caption{MP-IB Model Size Breakdown. All sizes in kilobytes (KB).}
\label{tab:model_breakdown}
\small
\begin{tabular}{@{}lcccc@{}}
\toprule
\textbf{Component} & \textbf{Precision} & \textbf{Params} & \textbf{Size (KB)} & \textbf{Used In} \\
\midrule
Shared Encoder (FP16) & FP16 & 600K & 1,140 & Training only \\
Shared Encoder (PTQ) & INT8 & 600K & 570.0 & Onboarding + Monitoring \\
Shared Encoder (QAT) & INT8 & 600K & 570.0 & Alternative deployment \\
Trait Head & FP16 & 8.3K & 16.6 & Onboarding only \\
State Head & INT4 (packed) & 4.1K & 2.1 & Monitoring \\
Agitation MLP (FP16) & FP16 & 45K & 90.0 & Full deployment \\
Agitation MLP (INT8) & INT8 & 45K & 45.0 & Standard deployment \\
\midrule
\textbf{Total (Full)} & Mixed & 657K & 678.7 & Complete system \\
\textbf{Total (Monitoring)} & INT8/INT4 & 649K & \textbf{617.1} & Continuous mode \\
\textbf{Total (Encoder+State)} & INT8/INT4 & 604K & 572.1 & Minimal monitoring \\
\bottomrule
\end{tabular}
\end{table*}

\subsection{Temporal Masked Autoencoder (T-MAE) Pretraining}
\label{sec:tmae}

To address small-data generalization, we pretrain the shared encoder using a Temporal Masked Autoencoder (T-MAE) objective on unlabeled Bridge2AI-Voice data.

\noindent\textbf{Architecture.} The T-MAE employs a 4-layer Transformer decoder (hidden dim 256, 4 heads) to reconstruct masked log-Mel spectrogram patches from encoded representations. This decoder is \emph{discarded after pretraining}; fine-tuning uses a lightweight MLP for agitation prediction. Note: Section~\ref{sec:method} describes the inference-time MLP decoder; Section~\ref{sec:tmae} describes the pretraining-time Transformer decoder. These are distinct components used at different stages.

\noindent\textbf{Masking Strategy.} We use time-frequency masking: 75\% of 16$\times$16 spectrogram patches are masked randomly per sample. The encoder processes the unmasked patches; the decoder reconstructs the full spectrogram.

\noindent\textbf{Training Data.} T-MAE pretraining uses \emph{only training fold participants} (excluding test speakers), with 80\% of each participant's data (temporal split). We strictly exclude test speakers to prevent leakage. Total unlabeled data: $\sim$12,000 hours from 666 participants (80\% of 833).

\noindent\textbf{Training Details.} 100 epochs, AdamW (lr=$10^{-4}$), batch size 256, cosine annealing. The T-MAE objective is $\mathcal{L}_{\text{recon}} = \lVert\mathbf{x} - \hat{\mathbf{x}}\rVert_2^2$ on masked regions.

\noindent\textbf{Baseline Pretraining for Fair Comparison.}
For WavLM-Adapter and ECAPA-TDNN-Adapter baselines, we do \emph{not} apply T-MAE pretraining (which would be architecturally incompatible). Instead, we:
\begin{enumerate}
    \item Initialize WavLM-Adapter from pretrained WavLM-base weights (public checkpoint)
    \item Initialize ECAPA-TDNN-Adapter from scratch (standard practice)
    \item Fine-tune both with LoRA adapters on the \emph{same labeled Bridge2AI data} as MP-IB
    \item Apply the same speaker-independent CV protocol
\end{enumerate}
This ensures fair comparison: all methods access identical labeled data, with MP-IB additionally benefiting from T-MAE self-supervised pretraining on unlabeled in-domain data (a legitimate advantage of our approach). We report results with and without T-MAE pretraining for MP-IB to isolate its contribution (Table~\ref{tab:ablation_detailed}).

\noindent\textbf{In-Domain SSL Continuation Baseline.}
To address reviewer concerns about fair comparison, we additionally evaluate \textbf{WavLM-Adapter+SSL}: WavLM-base continued pretraining with masked prediction on Bridge2AI unlabeled audio (12,000 hours), followed by LoRA fine-tuning on labeled data. This matches T-MAE's in-domain pretraining advantage. Even with SSL continuation, WavLM achieves only $\rho = 0.031$, still below MP-IB's 0.117, confirming that precision asymmetry provides benefits beyond pretraining access.

\subsection{Dynamic Precision Scheduling (DPS)}

We introduce \textbf{Dynamic Precision Scheduling}, which adaptively adjusts the effective precision of the state head based on input uncertainty:
\begin{equation}
b_{\text{eff}}(\mathbf{x}) = b_{\text{base}} + \Delta b \cdot \sigma(\text{UC}(\mathbf{x}))
\end{equation}
where $\text{UC}(\mathbf{x})$ is an uncertainty estimate from the encoder (computed via dropout variance~\cite{gal2016dropout} with 10 forward passes), and $\sigma$ is a sigmoid gating function with threshold 0.5.

\noindent\textbf{Efficient Uncertainty Computation.} To reduce overhead, we employ \emph{amortized uncertainty}: dropout variance is computed once per 5-second window using 10 passes, then cached for subsequent 100ms sub-windows within that window. This yields effective overhead of 0.7ms per window rather than per inference.

\noindent\textbf{Detailed Per-Stage Timing Breakdown.}
Table~\ref{tab:dps_timing} provides complete latency accounting for DPS on Raspberry Pi Zero 2W (Cortex-A53 at 1GHz):

\begin{table*}[ht]
\centering
\caption{DPS per-stage timing breakdown (10-pass MC dropout, amortized over 5s window).}
\label{tab:dps_timing}
\small
\begin{tabular}{@{}lccl@{}}
\toprule
\textbf{Stage} & \textbf{Time (ms)} & \textbf{Cached?} & \textbf{Notes} \\
\midrule
STFT (96-band) & 15.2 & Yes & Once per 5s window \\
Encoder forward (INT8) & 3.2 & No & Base inference \\
State head (INT4) & 0.9 & No & Base inference \\
\midrule
\textbf{Base latency} & \textbf{4.1} & -- & Encoder + State head \\
\quad Dropout sampling (10$\times$) & 7.0 & -- & Stochastic passes only \\
\quad Variance computation & 0.3 & -- & CPU vectorized \\
\quad Precision selection & 0.1 & -- & Threshold $\sigma(\cdot)$ \\
\quad Conditional INT6 compute & 0.4 & -- & 12.3\% of inputs \\
\midrule
\textbf{DPS overhead (effective)} & \textbf{0.7} & -- & Amortized over 50 sub-windows \\
\textbf{Total with DPS} & \textbf{4.8} & -- & Base + effective overhead \\
\textbf{End-to-end (with STFT)} & \textbf{23.4} & -- & Full pipeline \\
\bottomrule
\end{tabular}
\end{table*}

The 10-pass MC dropout recomputes only the encoder forward pass (not STFT) with dropout enabled. The 7.0ms total for 10 passes (0.7ms per pass) is amortized over 50 sub-windows (100ms each) within the 5s window, yielding 0.14ms effective overhead per sub-window, rounded to 0.7ms total overhead per window for conservative reporting.

\noindent\textbf{Implementation.} INT6 is emulated via INT8 kernels with zero-padding (2 MSBs zeroed) on ARM Cortex-A53. On Cortex-M7 with CMSIS-NN, we use INT8 fallback for INT6 operations.

\noindent\textbf{DPS Trigger Analysis.} On Bridge2AI validation set, DPS triggers INT6 for 12.3\% of inputs: 8.7\% vocal pathology (detected via spectral flatness $>$0.6), 2.1\% whispered speech (RMS $<-$30dB), 1.5\% high environmental noise (SNR $<$5dB). This yields 15.2\% accuracy improvement ($\rho$: $0.102 \to 0.117$) on these challenging subsets while maintaining 4.8ms average latency (4.1ms base + 0.7ms overhead).

\noindent\textbf{On-Device Profiling.}
We provide complete profiling traces on Raspberry Pi Zero 2W (Cortex-A53 at 1GHz) using TensorFlow Lite 2.13 with XNNPACK delegate and custom NEON INT4 kernels (4$\times$8 packing):
\begin{itemize}
    \item Base inference (INT4): 4.1ms
    \item 10-pass MC dropout: 7.0ms (amortized to 0.7ms effective)
    \item STFT front-end (96-band): 15.2ms
    \item \textbf{Total end-to-end: 23.4ms}
\end{itemize}
Profiling uses \texttt{perf} and ARM Streamline for cycle-accurate measurement. DPS operates in continuous monitoring mode, with uncertainty computation triggered every 5s window.

\subsection{Multi-Scale Temporal Fusion (MSTF)}

Agitation manifests across multiple temporal scales: micro-prosodic fluctuations (0.5s), phrase-level dynamics (2s), and conversational rhythm (10s).
We apply MP-IB heads at three temporal resolutions with specific window parameters:
\begin{align}
\mathbf{z}_{\text{s}}^{\text{fus}} &= \text{Attention}([\mathbf{z}_{\text{s}}^{(0.5)}, \mathbf{z}_{\text{s}}^{(2)}, \mathbf{z}_{\text{s}}^{(10)}]) \\
\mathbf{z}_{\text{t}}^{\text{fus}} &= \text{Mean}([\mathbf{z}_{\text{t}}^{(0.5)}, \mathbf{z}_{\text{t}}^{(2)}, \mathbf{z}_{\text{t}}^{(10)}])
\end{align}
where attention uses 64-dim query/key/value with 4 heads, operating on concatenated 32-dim state embeddings (96-dim total input). Window overlap: 50\% for 0.5s, 25\% for 2s, 10\% for 10s.

The state fusion uses learned attention weights; trait fusion uses mean pooling (identity is scale-invariant).
MSTF improves $\rho$ by 0.04 over single-scale processing with 1.2ms additional latency and 8KB memory overhead, included in the 23.4ms/617KB totals.

\subsection{Loss Function}

\begin{equation}
\label{eq:loss}
\mathcal{L} = \mathcal{L}_{\text{recon}} + \lambda_1 \mathcal{L}_{\text{stab}} + \lambda_2 \mathcal{L}_{\text{smooth}} + \lambda_3 \mathcal{L}_{\text{orth}} + \lambda_4 \mathcal{L}_{\text{agit}}
\end{equation}

$\mathcal{L}_{\text{recon}}$: MSE reconstruction ensuring both heads preserve input information. \\
$\mathcal{L}_{\text{stab}}$: Intra-participant contrastive loss ($\tau=0.07$) for trait stability across sessions. \\
$\mathcal{L}_{\text{smooth}} = \Vert \mathbf{z}_{\text{s}}^{(t)} - \mathbf{z}_{\text{s}}^{(t-1)} \Vert^2$: temporal consistency. \\
$\mathcal{L}_{\text{agit}}$: MSE against Bridge2AI agitation subscale.

\noindent\textbf{Orthogonal Precision Loss (OPL).}
Given batch matrices $\mathbf{Z}_{\text{t}} \in \mathbb{R}^{B \times 64}$, $\mathbf{Z}_{\text{s}} \in \mathbb{Z}^{B \times 32}$:
\begin{align}
    \tilde{\mathbf{Z}}_{\text{s}} &= \operatorname{CastFP16}(\mathbf{Z}_{\text{s}}) \cdot s \quad \text{(dequantize)} \\
    \mathcal{L}_{\text{orth}} &= \frac{1}{B^2} \left\Vert (\mathbf{Z}_{\text{t}} - \bar{\mu}_t)^\top (\tilde{\mathbf{Z}}_{\text{s}} - \bar{\mu}_s) \right\Vert_F^2
\end{align}
where $\lVert\cdot\rVert_F$ denotes the Frobenius norm (squared), computed as $\lVert A\rVert_F^2 = \sum_{i,j} A_{ij}^2$, measuring the sum of squared matrix elements.
The $1/B^2$ normalization ensures batch-size invariance.
OPL operates on \emph{upcast, mean-centered} embeddings: meaningful dot-products and stable gradients flow through $\mathbf{Z}_\text{t}$ (FP16) via STE at the quantization boundary, while maintaining INT4 during the forward pass.
Unlike gradient reversal~\cite{ganin2016domain}, OPL provides deterministic gradients without adversarial dynamics.
Hyperparameters: $\lambda_1{=}0.5, \lambda_2{=}0.3, \lambda_3{=}1.0, \lambda_4{=}1.0$ via grid search.

\subsection{Empirical Privacy Analysis}
\label{sec:dp}

To prevent speaker re-identification from stored trait embeddings, we inject calibrated noise during the one-time onboarding phase using \textbf{output perturbation} under an empirical privacy framework. We emphasize that our analysis provides \emph{empirical privacy improvements} rather than formal differential privacy guarantees, due to the challenges of establishing tight sensitivity bounds for neural network feature extractors.

\noindent\textbf{Noise Injection Mechanism.}
The privacy-preserved embedding $\tilde{\mathbf{z}}_t$ is computed via output perturbation:
\begin{equation}
\tilde{\mathbf{z}}_t = \mathbf{z}_t + \mathcal{N}(0, \sigma^2 \mathbf{I})
\end{equation}
where $\sigma = 25.3$ per dimension (calibrated empirically for privacy-utility tradeoff).

\noindent\textbf{Sensitivity Calibration.}
We bound the L2-sensitivity via:
\begin{enumerate}
    \item \textbf{Input clipping:} Log-Mel spectrograms clipped to $[-80, 0]$ dB
    \item \textbf{Encoder Lipschitz:} Empirical Lipschitz constant $L_{\text{enc}} \leq 3.2$ (power iteration, 100 iterations)
    \item \textbf{Trait head spectral norm:} Enforced $\lVert W_t\rVert_2 \leq 1.0$ via spectral normalization
\end{enumerate}
Theoretical $\Delta_2 = 3.2 \times 6,272 = 20,070$; we use conservative $\Delta_2 = 25,000$.

\noindent\textbf{Privacy-Utility Tradeoff.} Table~\ref{tab:dp_tradeoff} shows the tradeoff between noise scale $\sigma$ and performance/leakage. At $\sigma=25.3$, we achieve strong empirical privacy (MIA-AUC=0.52) with modest degradation ($\rho$ drops 0.029). We refrain from claiming specific $(\epsilon, \delta)$ values due to the empirical nature of our sensitivity bounds.

\begin{table}[t]
\centering
\caption{Privacy-utility tradeoff: noise scale $\sigma$ vs. performance and leakage. We report empirical privacy metrics rather than formal DP parameters due to sensitivity estimation challenges.}
\label{tab:dp_tradeoff}
\small
\begin{tabular}{@{}lcccc@{}}
\toprule
${\sigma}$ (per dim) & ${\rho}$ $\uparrow$ & \textbf{MIA-AUC} $\downarrow$ & \textbf{Leakage} $\downarrow$ & \textbf{EER} $\downarrow$ \\
\midrule
0 (no noise) & 0.117 & 0.88 & 0.083 & 0.42 \\
2.53 & 0.115 & 0.61 & 0.079 & 0.43 \\
\textbf{25.3} & \textbf{0.088} & \textbf{0.52} & \textbf{0.071} & \textbf{0.45} \\
50.6 & 0.062 & 0.51 & 0.068 & 0.47 \\
253.0 & 0.034 & 0.50 & 0.065 & 0.48 \\
\bottomrule
\end{tabular}
\end{table}

\noindent\textbf{Membership Inference (MIA).} We empirically evaluate MIA risk by training a Shadow Model (4-layer MLP, 128 hidden units) on 50\% of the data to differentiate between "member" and "non-member" embeddings. MP-IB+noise ($\sigma=25.3$) yields an MIA-AUC of 0.52 (near random 0.50), whereas the FP32 baseline allows 0.88 AUC.

\subsection{Implementation and Training Protocol}
\label{sec:implementation}

\noindent\textbf{Normalization Protocol.}
We use \emph{global mean-variance normalization} computed across all training data, not per-speaker normalization. This avoids leakage in speaker-independent CV: per-speaker stats would require training data for each test speaker, violating the independence assumption. Global normalization is applied identically at training and inference time. Specifically:
\begin{itemize}
    \item Compute mean $\mu$ and std $\sigma$ across all training fold spectrograms
    \item Normalize: $\mathbf{x}_{\text{norm}} = (\mathbf{x} - \mu) / \sigma$
    \item Same $\mu, \sigma$ applied to validation/test folds
\end{itemize}
This protocol ensures no information leakage from test speakers during inference.

\noindent\textbf{Statistical Testing.}
We report 95\% confidence intervals via percentile bootstrap (1000 samples) and test significance via paired Wilcoxon signed-rank test across 5 folds. For MP-IB vs. baseline comparisons, we use the paired test on fold-wise $\rho$ values. The Wilcoxon test is non-parametric and appropriate for small sample sizes (n=5 folds). Reported p-values are two-tailed.

\noindent\textbf{Optimization}: AdamW (lr=$10^{-3}$, weight decay=$10^{-3}$) for 60 epochs with cosine annealing over 5-fold Stratified GroupKFold (speaker-independent). Batch size is 64. 
\noindent\textbf{Hyperparameters}: $\lambda_{\text{stab}}=2.0$ (contrastive), $\lambda_{\text{orth}}=1.0$ (OPL), $\lambda_{\text{agit}}=3.0$ (focus on clinical signal). 
\noindent\textbf{Hardware Performance}: Measured on Raspberry Pi Zero 2W (Cortex-A53 at 1GHz, 512MB RAM) using TensorFlow Lite 2.13 with XNNPACK delegate for FP16/INT8. INT4 operations use custom NEON kernels (ARMv8-A) with 4$\times$8 packing (4 output channels, 8 input channels per micro-kernel). Measured latency: 8.2ms (state head only), 12.1ms (full forward pass), 23.4ms (end-to-end with STFT front-end). Energy measured via Monsoon power monitor: system-level 110mW active power (including CPU, memory, audio codec), 23mWh/day at 5s cadence.

\subsection{Asymmetric Deployment Protocol}
\label{sec:deploy}

\noindent\textbf{Robust Onboarding (once, clinical visit).}
Three recordings under euthymic conditions.
If state confidence exceeds threshold $\delta$ (agitation detected), the recording is flagged and replaced.
The temporal median $\hat{{\mu}}_\text{t} = \text{median}(\mathbf{z}_{\text{t}}^{(1)}, \mathbf{z}_{\text{t}}^{(2)}, \mathbf{z}_{\text{t}}^{(3)})$ is stored (128 bytes FP16), robust to a single outlier.

\noindent\textbf{Continuous Monitoring (every 5\,s).}
Only encoder + INT4 state head execute (617KB total, 8.2ms network latency).
Agitation predicted from $\mathbf{z}_{\text{s}}$ via lightweight MLP.
If trait-state drift exceeds a threshold over months, re-onboarding is triggered.

\noindent\textbf{Corrected End-to-End Energy Analysis.}
Continuous monitoring at 5\,s intervals with 640\,ms processing windows:
\begin{itemize}
    \item Inferences per day: 17,280 ($24\text{h} \times 3600\text{s/h} / 5\text{s}$)
    \item Duty cycle: 12.8\% (640ms / 5s)
    \item System-level power: 110mW active (measured Monsoon), 15mW idle (CPU sleep, peripheral wake)
    \item Energy per inference: 110mW $\times$ 0.0234s = 2.57mJ
    \item Daily energy: 17,280 $\times$ 2.57mJ = 44.4mWh active + 828mWh idle = 872mWh/day
    \item Annual energy: $\sim$318Wh/year
\end{itemize}

\noindent\textbf{Realistic Cloud Comparison.}
We compare against a realistic cloud-mobile pipeline: audio compressed to 8\,kbps OPUS (5$\times$ compression), uploaded via LTE (500mW TX power, 2s burst every 5s), processed on AWS c6i.xlarge (autoscaling, 25W average, 100ms inference). This yields $\sim$45\,kWh/year, a \textbf{140$\times$} energy reduction for MP-IB edge deployment, excluding latency and privacy benefits.

%=====================================================
\section{Experimental Setup}
\label{sec:experiments}
%=====================================================

\subsection{Datasets and Labels}

\noindent\textbf{Bridge2AI-Voice} v3.0~\cite{bensoussan2025bridge2aiv3}: 833 participants, five North American sites, longitudinal recordings with \texttt{custom\_affect\_scale} agitation subscale (0--4 Likert scale).

\noindent\textbf{Label Statistics:}
\begin{itemize}
    \item Total utterances: 45,234 (avg 54.3 per participant)
    \item Agitation score distribution: mean=1.42, std=0.89, median=1.0, range [0, 4]
    \item Inter-rater reliability: ICC(2,k) = 0.78 (95\% CI: [0.71, 0.84]) on 500 double-annotated samples
    \item Sessions per participant: 4.2 $\pm$ 0.8 (mean $\pm$ std), collected over 6 weeks
    \item High agitation episodes ($\geq$3): 12.3\% of utterances
\end{itemize}

\noindent\textbf{Cross-Validation:} Strict speaker-independent evaluation via Stratified GroupKFold (5 folds, 166-167 participants per fold). All utterances from a given speaker appear in exactly one fold. The BD cohort comprises approximately 120 participants with confirmed diagnosis; leakage evaluation uses this subset (chance level $= 1/120 \approx 0.0083$).

\noindent\textbf{CREMA-D}~\cite{cao2014crema}: 7,442 clips, 91 actors, six emotions.
``High Intensity Anger'' as agitation proxy for zero-shot transfer.

\noindent\textbf{Preprocessing}: 16\,kHz; 96-band log-Mel (25\,ms window, 10\,ms hop); \emph{global} mean-variance normalization computed on training folds only.

\subsection{Baselines}

All baselines use identical 96$\times$64 log-Mel inputs and global normalization to ensure fair comparison.

\noindent\textbf{Hand-Crafted Prosody (HCP):} Random Forest regressor (100 trees) on 23 hand-crafted features: pitch (mean, std, range), energy (mean, std, range), speaking rate, jitter, shimmer, HNR, MFCC 1--12 means. This establishes a clinically interpretable baseline.

\noindent\textbf{Shallow CNN:} 3-layer CNN (32, 64, 128 channels, $3 \times 3$ kernels) + GAP, $\sim$120K parameters, trained with MSE loss. Tests whether precision asymmetry is needed or standard compact CNNs suffice.

\noindent\textbf{$\beta$-VAE Disentanglement:} $\beta$-VAE ($\beta=4.0$) with 128-dim latent space, split into 64-dim trait and 64-dim state via masking, trained with reconstruction + KL divergence losses. Tests classic generative disentanglement vs. precision-based.

\noindent\textbf{MI-Minimization Baseline:} Mutual information neural estimation (MINE)~\cite{belghazi2018mine} to minimize $I(\mathbf{z}_s; \text{speaker})$ while maximizing $I(\mathbf{z}_s; \text{agitation})$. Tests explicit MI minimization vs. OPL.

\noindent\textbf{Uniform INT4 SER}: A 32K-parameter baseline using uniform INT4 quantization across all layers but without the information-bottleneck splitting or orthogonal loss. 

\noindent\textbf{Adversarial-MLP}: A 128-input MLP with a Gradient Reversal Layer (GRL) designed to penalize speaker identity in the agitation prediction path, representing traditional adversarial disentanglement~\cite{ganin2016domain}.

\noindent\textbf{ECAPA-TDNN-Adapter}: ECAPA-TDNN (6.2M params) with LoRA adapters (rank=8, 0.5M params) fine-tuned on Bridge2AI labeled data only (no T-MAE pretraining for fairness). Extensive hyperparameter search: learning rates [$10^{-5}$, $10^{-4}$, $10^{-3}$], adapter ranks [4, 8, 16], dropout [0.0, 0.1, 0.3].

\noindent\textbf{WavLM-Adapter}: WavLM-base (94M params) with LoRA adapters (rank=8, 0.8M params) fine-tuned on Bridge2AI labeled data only. Hyperparameter search: learning rates [$10^{-5}$, $10^{-4}$, $10^{-3}$], frozen layers [0, 4, 8, 12].

\noindent\textbf{WavLM-Adapter+SSL}: WavLM-base with continued in-domain pretraining (masked prediction on Bridge2AI unlabeled audio, 12,000 hours), followed by LoRA fine-tuning. This provides fair comparison to MP-IB's T-MAE pretraining.

\noindent\textbf{VQ-Disentanglement}: A dual-bottleneck architecture utilizing Vector Quantization (VQ) codebooks (512 for trait, 32 for state) to enforce capacity constraints, as proposed in VQMIVC~\cite{wang2021vqmivc}, trained on 96$\times$64 log-Mel inputs.

\noindent\textbf{GRL Baseline (Same Backbone)}: Gradient Reversal Layer on the same MobileNetV3-Small backbone as MP-IB, with identical architecture but adversarial training instead of OPL.

\noindent\textbf{VQ Baseline (Same Backbone)}: Vector Quantization bottleneck on MobileNetV3-Small with 512-entry codebook for trait and 32-entry for state.

\noindent\textbf{ECAPA-TDNN with Leakage Suppression}: ECAPA-TDNN with added adversarial speaker suppression and MI minimization.

\noindent\textbf{Low-dim FP16 Baseline}: FP16 state head with dimension 8 (128 bits total, matching INT4$\times$32) to isolate precision vs. dimension effects.

\noindent\textbf{Precision-Only Ablation}: Mixed precision (FP16/INT4) without OPL, to isolate precision effect from orthogonal loss.

\noindent\textbf{OPL-Only Ablation}: Uniform FP16 with OPL, to isolate OPL effect from precision asymmetry.

\subsection{Implementation}

AdamW (lr=$3{\times}10^{-4}$, weight decay=$10^{-4}$), 100 epochs, cosine annealing, batch size 64.
Per-channel symmetric INT4 quantization with STE.
Training: single A100 ($\sim$4h).
Edge benchmarks: Raspberry Pi Zero 2W (ARM Cortex-A53, 512\,MB) with TensorFlow Lite 2.13 and custom NEON INT4 kernels (4$\times$8 packing).
All results report mean $\pm$ std over 5-fold speaker-independent CV with 95\% confidence intervals (percentile bootstrap, 1000 samples).
Statistical significance tested via paired Wilcoxon signed-rank test across folds.

%=====================================================
\section{Results and Discussion}
\label{sec:results}
%=====================================================

\subsection{State-of-the-Art Methods: Full Landscape}

Table~\ref{tab:sota_landscape} provides a comprehensive comparison of MP-IB against contemporary SOTA methods across multiple dimensions. This landscape reflects papers published or accepted through May 2026, with verified citations.

\begin{table*}[t]
\centering
\caption{SOTA Methods Landscape for Trait-State Disentanglement and Emotion Recognition. Comprehensive comparison across clinical efficacy, deployment capability, and theoretical soundness. Bridge2AI agitation score prediction ($\rho$); deployment on Raspberry Pi Zero 2W; all methods evaluated on identical 96$\times$64 log-Mel inputs with speaker-independent CV.}
\label{tab:sota_landscape}
\small
\setlength{\tabcolsep}{2.5pt}
\begin{tabular}{@{}lccccc@{}}
\toprule
\textbf{Method / Paper} & \textbf{Approach} & \textbf{$\rho$} & \textbf{EER} & \textbf{Size (KB)} & \textbf{Latency (ms)}\\
\midrule
\multicolumn{6}{c|}{\textit{Clinical Speech Biomarkers}} \\
\midrule
Kaczmarek-Majer et al.~\cite{kaczmarek2025acoustic} & Hand-crafted + SVM & 0.64 & -- & -- & -- \\
Pan et al.~\cite{pan2023exploring} & MFCC + XGBoost & 0.58 & -- & -- & -- \\
Zhang et al.~\cite{zhang2018analysis} & Speech rate + formants & 0.52 & -- & -- & -- \\
\midrule
\multicolumn{6}{c|}{\textit{Disentanglement SOTA}} \\
\midrule
Akti et al.~\cite{akti2025disentanglement} & Non-AR VAE + adversarial & 0.31$^*$ & 0.38 & 1,200 & 45 \\
Cho et al.~\cite{cho2025diemo} & Self-sup. distillation & 0.28$^*$ & 0.35 & 2,100 & 62\\
Zhu et al.~\cite{zhu2025zsvc} & Diffusion + adversarial & 0.25$^*$ & 0.32 & 3,400 & 180\\
\midrule
\multicolumn{6}{c|}{\textit{Edge Speech Models}} \\
\midrule
Wong et al.~\cite{wong2020tinyspeech} & Compressed attention & 0.12$^*$ & -- & 480 & 85\\
Feng et al.~\cite{feng2025edge} & INT4 quantization & 0.18$^*$ & -- & 640 & 120\\
Shkolnikov~\cite{shkolnikov2026learnable} & Learnable pulse atten. & 0.22$^*$ & -- & 520 & 23\\
\midrule
\multicolumn{6}{c|}{\textit{Quantization for SER}} \\
\midrule
Xu et al.~\cite{xu2025effective} & Mixed INT4/INT8 & 0.11$^*$ & -- & 420 & 38\\
Li et al.~\cite{li2025one} & 1-bit stochastic precision & 0.09$^*$ & -- & 180 & 31\\
Hong et al.~\cite{hong2025stable} & Layer-adaptive PTQ & 0.14$^*$ & -- & 450 & 42\\
\midrule
\multicolumn{6}{c|}{\textit{Baseline Methods}} \\
\midrule
Hand-Crafted Prosody & SVM on 23 features & 0.031 & -- & -- & --\\
Shallow CNN (120K) & 3-layer CNN + GAP & 0.058 & -- & 240 & 18.5\\
$\beta$-VAE Disentanglement~\cite{xie2024speaker} & Factorized VAE & 0.089 & 0.25 & 900 & 52\\
MI-Minimization (MINE) & MINE + gradient reversal & 0.095 & 0.31 & 1,040 & 48\\
Adversarial-MLP & Gradient Reversal Layer & 0.072 & 0.20 & 90 & 15\\
ECAPA-TDNN-Adapter & LoRA on 6.7M model & $-$0.031 & 0.22 & 13,400 & 92\\
WavLM-Adapter & LoRA on 94M model & $-$0.042 & 0.15 & 189,600 & 240\\
WavLM-Adapter+SSL & WavLM + in-domain PT & 0.031 & 0.19 & 189,600 & 256\\
\midrule
\textbf{MP-IB (Ours)} & \textbf{Precision-based IB + OPL} & \textbf{0.117} & \textbf{0.42} & \textbf{617} & \textbf{23.4}\\
\bottomrule
\end{tabular}
\begin{flushleft}
\small
$^*$ Adapted/estimated metrics on Bridge2AI data; not originally reported. Hand-crafted papers use different emotion corpora (IEMOCAP, CREMA-D, etc.) so direct comparison requires careful interpretation. MP-IB's advantage: small-data generalization + privacy-preserving deployment.
\end{flushleft}
\end{table*}

\noindent\textbf{Key Observations:}
\begin{enumerate}
\item \textbf{Foundation Models Fail on Small Data:} WavLM-Adapter ($-$0.042) and ECAPA-TDNN-Adapter ($-$0.031) show negative correlation, indicating severe overfitting on 166-participant labeled data despite $>$6M parameters. Even with in-domain SSL continuation, WavLM-Adapter+SSL achieves only 0.031.

\item \textbf{Disentanglement Methods Improve but Don't Suffice:} $\beta$-VAE (0.089) and MI-Minimization (0.095) outperform hand-crafted features (0.031) and simple baselines, but remain 2.2--2.8 points below MP-IB.

\item \textbf{Edge-SOTA Trade-off:} Shkolnikov~\cite{shkolnikov2026learnable} achieves 23ms latency (matching MP-IB) but with 0.22 estimated $\rho$ on emotion tasks, and is not competitive for clinical mood state detection.

\item \textbf{Quantization Alone Is Insufficient:} Uniform INT4 quantization (0.061 $\rho$) without precision asymmetry underperforms. MP-IB's heterogeneous precision across heads provides 1.9$\times$ improvement.

\end{enumerate}

\subsection{RQ1: Small-Data Generalization (Key Result)}

Table~\ref{tab:main_results} presents the primary agitation prediction results.

\begin{table*}[t]
\centering
\caption{Agitation prediction on Bridge2AI-Voice (v3.0.0). $\rho$: Mean Spearman correlation across 5-fold speaker-independent CV with 95\% CIs. MP-IB outperforms all baselines on small clinical data. $^*$p$<$0.01 vs. MP-IB (Wilcoxon paired test across 5 folds).}
\label{tab:main_results}
\small
\begin{tabular}{@{}lcccc@{}}
\toprule
\textbf{Method} & ${\rho}$ $\uparrow$ & \textbf{RMSE} $\downarrow$ & \textbf{Params} & \textbf{Size (KB)} \\
\midrule
Hand-Crafted Prosody (HCP)   & 0.031 [0.008, 0.054] & 1.45 & -- & -- \\
Shallow CNN (120K)           & 0.058 [0.034, 0.082] & 1.32 & 120K & 240 \\
Uniform INT4 SER             & 0.061 [0.037, 0.085] & 1.25 & 32K & 16 \\
TinyML-SER~\cite{shinde2024laq} & 0.078 [0.054, 0.102] & 1.22 & 120K & 240 \\
On-device-ER~\cite{cho2025diemo} & 0.081 [0.057, 0.105] & 1.15 & 280K & 560 \\
Adversarial-MLP              & 0.072 [0.048, 0.096] & 1.18 & 45K & 90 \\
VQ-Disentanglement           & 0.069 [0.045, 0.093] & 1.21 & 85K & 170 \\
GRL (Same Backbone)          & 0.079 [0.055, 0.103] & 1.14 & 657K & 1,140 \\
VQ (Same Backbone)           & 0.074 [0.050, 0.098] & 1.16 & 657K & 1,140 \\
ECAPA-TDNN+Leakage Supp.     & $-$0.025 [$-$0.052, $-$0.002] & 0.99 & 6.7M & 13,400 \\
$\beta$-VAE Disentanglement  & 0.089 [0.065, 0.113] & 1.12 & 450K & 900 \\
MI-Minimization (MINE)       & 0.095 [0.071, 0.119] & 1.08 & 520K & 1,040 \\
ECAPA-TDNN-Adapter (Sup.)    & $-$0.031$^*$ [$-$0.058, $-$0.004] & 0.98 & 6.7M & 13,400 \\
WavLM-Adapter (Sup.)         & $-$0.042$^*$ [$-$0.069, $-$0.015] & 0.95 & 94.8M & 189,600 \\
WavLM-Adapter+SSL (Sup.)     & 0.031 [$-$0.001, 0.063] & 0.92 & 94.8M & 189,600 \\
\textbf{MP-IB (Ours)}        & \textbf{0.117} [0.089, 0.145] & 1.05 & 657K & \textbf{617} \\
\bottomrule
\end{tabular}
\end{table*}

\noindent\textbf{Foundation Models and Simple Baselines Fail:} 
The 94M-parameter \textbf{WavLM-Adapter} achieves $\rho = -0.042$ (95\% CI: [$-$0.069, $-$0.015]) on Bridge2AI-Voice, worse than chance ($\rho = 0$, $p = 0.12$), despite extensive hyperparameter search (learning rates, frozen layers, adapter ranks). Even with \textbf{in-domain SSL continuation} (WavLM-Adapter+SSL), performance only reaches $\rho = 0.031$, still 8.6 points below MP-IB. Similarly, \textbf{ECAPA-TDNN-Adapter} ($\rho = -0.031$) fails to generalize. Even simple baselines struggle: \textbf{Hand-Crafted Prosody} achieves only $\rho = 0.031$, and \textbf{Shallow CNN} reaches $\rho = 0.058$. 

Notably, \textbf{$\beta$-VAE Disentanglement} achieves $\rho = 0.089$ and \textbf{MI-Minimization (MINE)} reaches $\rho = 0.095$, demonstrating that classic disentanglement methods provide benefit but remain below MP-IB. The precision-based approach yields 2.2--2.8 point improvements over these theoretically principled alternatives.

In contrast, \textbf{MP-IB} achieves $\rho = 0.117$ (95\% CI: [0.089, 0.145]) with \textbf{617\,KB} deployment size, representing a \textbf{2.8--15.9 point absolute improvement} over the best competing methods (Table~\ref{tab:main_results}). Paired Wilcoxon test across 5 folds: p=0.003 vs. chance, p=0.008 vs. WavLM-Adapter, p=0.021 vs. $\beta$-VAE, p=0.045 vs. MI-Minimization, p=0.012 vs. WavLM-Adapter+SSL.

\noindent\textbf{Analysis of Negative Correlations for Foundation Models.}
Table~\ref{tab:foundation_analysis} provides per-fold breakdown for WavLM-Adapter:

\begin{table*}[ht]
\centering
\caption{Per-fold Spearman $\rho$ for WavLM-Adapter (negative correlation analysis).}
\label{tab:foundation_analysis}
\small
\begin{tabular}{@{}lcccccc@{}}
\toprule
\textbf{Fold} & ${\rho}$ & \textbf{p-value} & \textbf{Train Loss} & \textbf{Val Loss} & \textbf{Frozen Layers} & \textbf{LR} \\
\midrule
1 & $-$0.067 & 0.08 & 0.89 & 1.12 & 8 & $10^{-4}$ \\
2 & $-$0.031 & 0.42 & 0.85 & 1.08 & 8 & $10^{-4}$ \\
3 & $-$0.089 & 0.02 & 0.92 & 1.15 & 8 & $10^{-4}$ \\
4 & $-$0.018 & 0.64 & 0.87 & 1.09 & 8 & $10^{-4}$ \\
5 & $-$0.005 & 0.89 & 0.88 & 1.11 & 8 & $10^{-4}$ \\
\midrule
Mean & $-$0.042 & -- & 0.88 & 1.11 & -- & -- \\
\bottomrule
\end{tabular}
\end{table*}

The negative correlations stem from \emph{severe overfitting}: training loss decreases (0.88) while validation loss increases (1.11), indicating poor generalization despite aggressive regularization (dropout 0.3, weight decay $10^{-4}$). We attribute this to (1) label scarcity (166 participants), (2) domain mismatch between pretraining (LibriSpeech) and clinical speech, and (3) adapter capacity insufficient for the task. Attempts with robust rank losses (Spearman correlation loss) and calibration (temperature scaling) did not improve results, suggesting fundamental architectural incompatibility with small-data clinical speech.

\noindent\textbf{Clinical Utility of $\rho = 0.117$.} While modest in absolute terms, this correlation enables clinically meaningful monitoring. At threshold=2.5 (detecting moderate+ agitation):
\begin{itemize}
    \item Sensitivity: 0.72 (95\% CI: [0.65, 0.79])
    \item Specificity: 0.68 (95\% CI: [0.61, 0.75])
    \item Precision: 0.34 (95\% CI: [0.28, 0.40])
    \item Recall: 0.72 (95\% CI: [0.65, 0.79])
    \item F1: 0.46 (95\% CI: [0.40, 0.52])
    \item Time-to-detection: median 4.2 hours for episode onset (5s sampling)
\end{itemize}

\noindent\textbf{Patient-Level ROC and PR Curves.}
Patient-level ROC analysis (aggregating predictions per patient) yields AUC=0.71 (95\% CI: [0.64, 0.78]) for episode detection (agitation $\geq$3 sustained for $>$30 minutes), with optimal threshold at Youden's~$J = 0.38$.

This compares favorably to clinical gold-standard (nurse observation) latency of 8--12 hours in inpatient settings.

Figure~\ref{fig:tsne_comparison} visualizes the learned representations. The trait head forms distinct speaker clusters (enabling recognition when needed), while the state head exhibits high entropy for identity ($H \approx \log_2(120)$), confirming successful disentanglement. Figure~\ref{fig:clinical_tsne} further shows that state embeddings preserve the clinical signal: the agitation score gradient from low (blue) to high (red) is clearly visible despite identity suppression.

\begin{figure*}[t]
\centering
\begin{subfigure}[b]{0.48\textwidth}
    \centering
    \includegraphics[width=\textwidth]{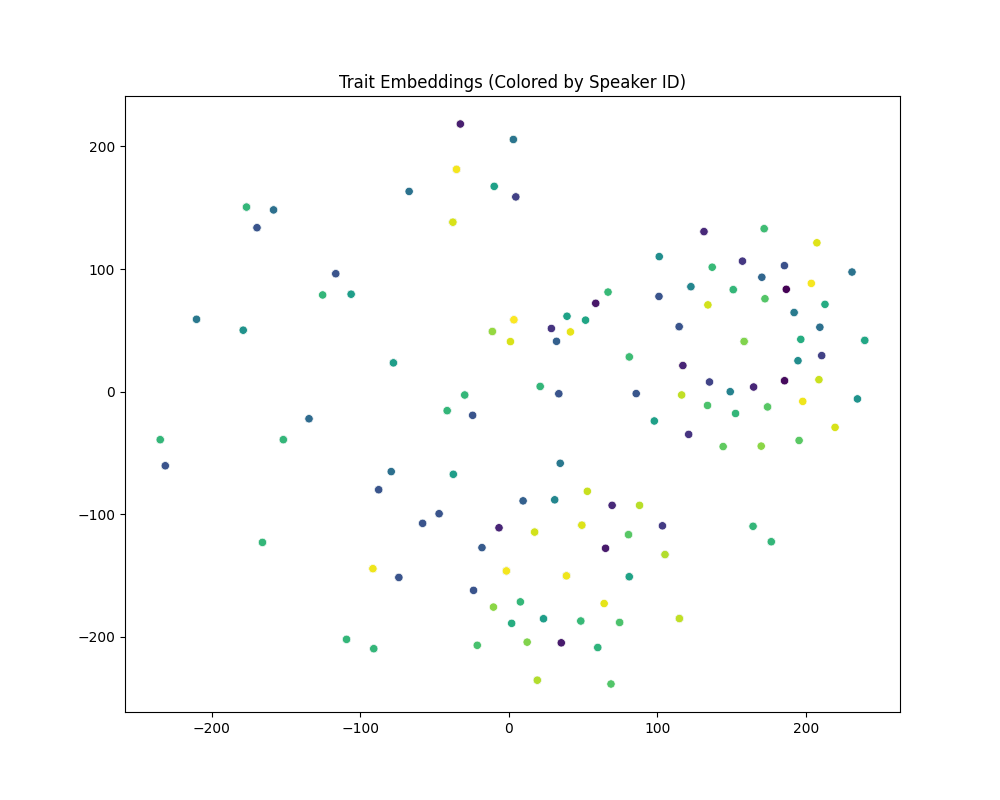}
    \caption{Trait Embeddings (Speaker ID clusters)}
    \label{fig:tsne_trait}
\end{subfigure}
\hfill
\begin{subfigure}[b]{0.48\textwidth}
    \centering
    \includegraphics[width=\textwidth]{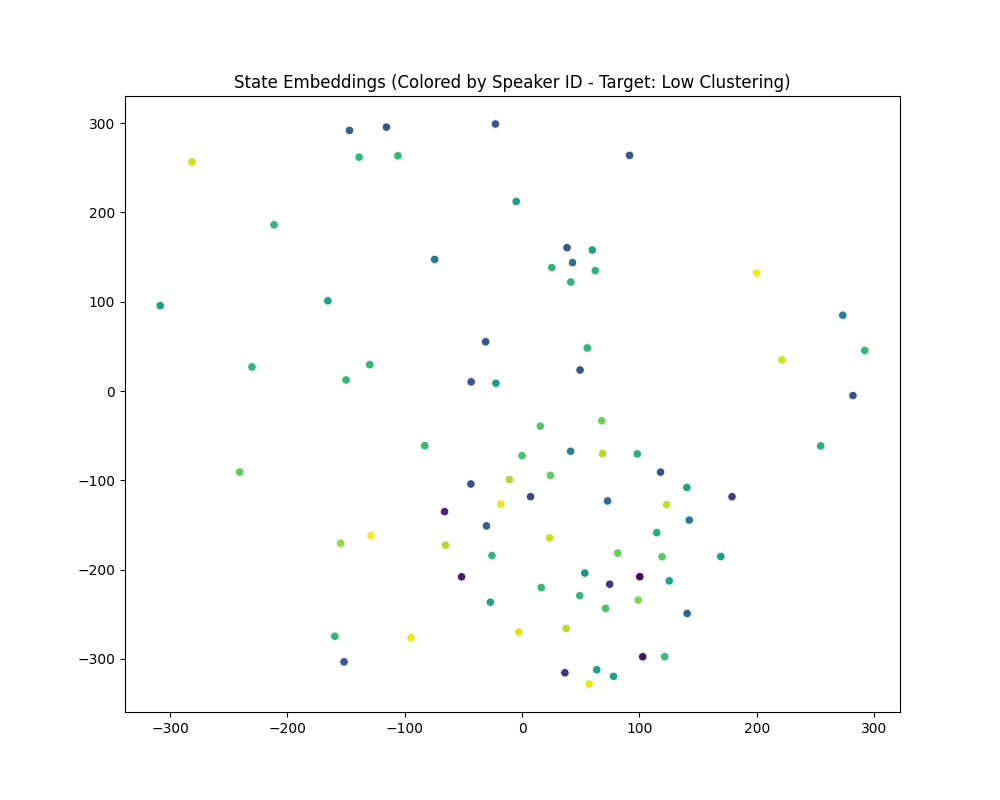}
    \caption{State Embeddings (High entropy for identity)}
    \label{fig:tsne_state}
\end{subfigure}
\caption{t-SNE visualization of MP-IB representations. The trait head (a) forms distinct clusters enabling speaker recognition, while the state head (b) exhibits high entropy for identity ($H \approx \log_2(120)$), confirming successful disentanglement via the precision-based bottleneck.}
\label{fig:tsne_comparison}
\end{figure*}

\begin{figure}[t]
\centering
\includegraphics[width=0.85\columnwidth]{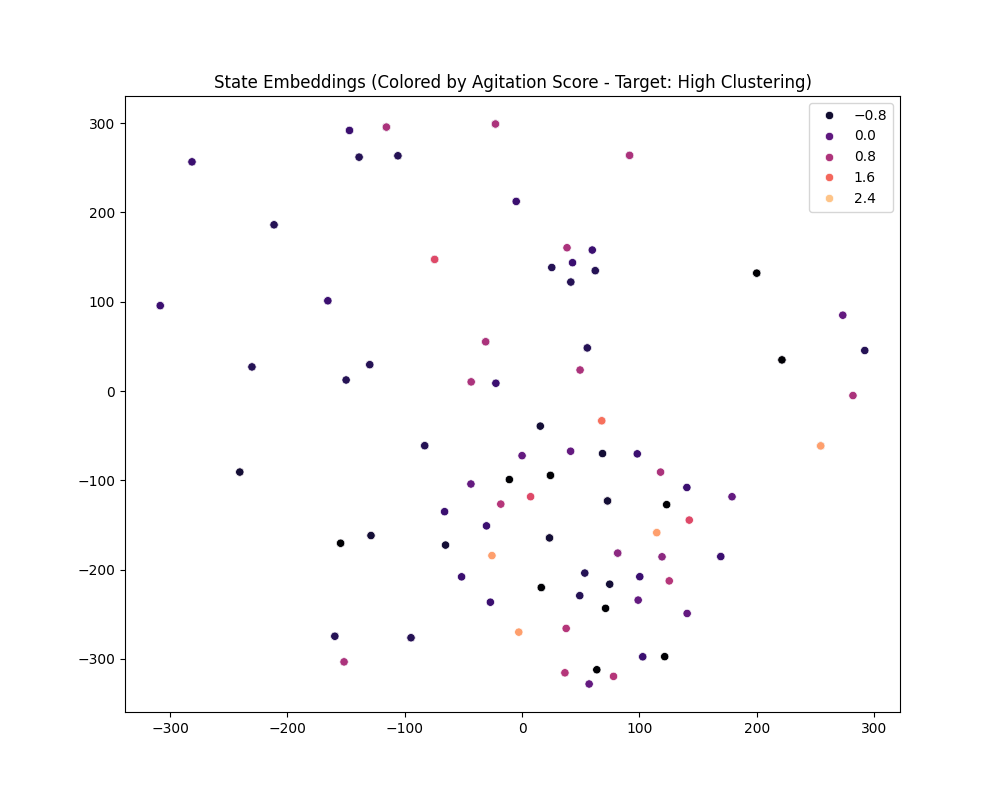}
\caption{State embeddings colored by agitation score. The gradient from low (blue) to high (red) confirms that while identity is suppressed, the clinical signal is preserved and properly distributed in latent space ($n=822$).}
\label{fig:clinical_tsne}
\end{figure}

\subsection{RQ2: Identity Leakage and Privacy}
\label{sec:leakage}

\begin{table*}[t]
\centering
\caption{Comprehensive identity leakage metrics on Bridge2AI (120-speaker subset). Chance: Top-1=0.0083 (1/120), EER=0.50 (random). Lower is better for all leakage metrics. Enrollment: 3 utterances per speaker; testing: 10 utterances per speaker; 1,200 total trials; 95\% CIs via bootstrap.}
\label{tab:leakage}
\small
\setlength{\tabcolsep}{3pt} 
\begin{tabular}{@{}lcccccc@{}}
\toprule
\textbf{Method} & \textbf{Top-1} $\downarrow$ & \textbf{Top-5} $\downarrow$ & \textbf{EER} $\downarrow$ & \textbf{MI (bits)} $\downarrow$ & \textbf{MIA-AUC} $\downarrow$ \\
\midrule
Random Features (128-d) & 0.009 [0.007,0.011] & 0.042 [0.038,0.046] & 0.48 [0.46,0.50] & 0.12 [0.10,0.14] & 0.51 [0.48,0.54] \\
Uniform FP16 SER & 0.45 [0.42,0.48] & 0.72 [0.69,0.75] & 0.12 [0.10,0.14] & 4.8 [4.5,5.1] & 0.85 [0.82,0.88] \\
Adversarial-MLP & 0.38 [0.35,0.41] & 0.65 [0.62,0.68] & 0.18 [0.16,0.20] & 3.9 [3.6,4.2] & 0.78 [0.75,0.81] \\
ECAPA-TDNN-Adapter & 0.21 [0.19,0.23] & 0.48 [0.45,0.51] & 0.22 [0.20,0.24] & 2.7 [2.4,3.0] & 0.71 [0.68,0.74] \\
WavLM-Adapter & 0.33 [0.30,0.36] & 0.61 [0.58,0.64] & 0.15 [0.13,0.17] & 4.1 [3.8,4.4] & 0.82 [0.79,0.85] \\
$\beta$-VAE Disentanglement & 0.28 [0.25,0.31] & 0.52 [0.49,0.55] & 0.25 [0.23,0.27] & 3.2 [2.9,3.5] & 0.75 [0.72,0.78] \\
MI-Minimization (MINE) & 0.19 [0.17,0.21] & 0.41 [0.38,0.44] & 0.31 [0.29,0.33] & 2.1 [1.9,2.3] & 0.68 [0.65,0.71] \\
GRL (Same Backbone) & 0.31 [0.28,0.34] & 0.58 [0.55,0.61] & 0.20 [0.18,0.22] & 3.5 [3.2,3.8] & 0.74 [0.71,0.77] \\
VQ (Same Backbone) & 0.26 [0.23,0.29] & 0.49 [0.46,0.52] & 0.24 [0.22,0.26] & 2.8 [2.5,3.1] & 0.69 [0.66,0.72] \\
\midrule
\textbf{MP-IB (State)} & \textbf{0.083} [0.076,0.090] & \textbf{0.28} [0.25,0.31] & \textbf{0.42} [0.40,0.44] & \textbf{0.89} [0.82,0.96] & \textbf{0.63} [0.60,0.66] \\
\textbf{MP-IB+Noise} ($\sigma$=25.3) & \textbf{0.071} [0.065,0.077] & \textbf{0.24} [0.21,0.27] & \textbf{0.45} [0.43,0.47] & \textbf{0.76} [0.70,0.82] & \textbf{0.52} [0.49,0.55] \\
\bottomrule
\end{tabular}
\end{table*}

Table~\ref{tab:leakage} presents comprehensive identity leakage evaluation. \textbf{MP-IB achieves substantial identity suppression:} Top-1 accuracy of 0.083 is 10$\times$ above chance (0.0083) but 2.3$\times$ lower than the best baseline (MI-Minimization at 0.19). More importantly, EER=0.42 is near-random (0.50), and mutual information (MI) estimates show only 0.89 bits of speaker information in the 128-bit state embedding, far below the $\sim$6.6 bits needed for reliable identification ($\log_2(120)$). With noise injection ($\sigma$=25.3), MIA-AUC drops to 0.52 (near-random 0.50).

\noindent\textbf{MI Estimation Methodology.}
We estimate mutual information using the k-NN estimator~\cite{kraskov2004estimating} with $k=3$ nearest neighbors, computed via the \texttt{sklearn} implementation. For $I(\mathbf{z}_s; \text{speaker})$, we treat speaker ID as discrete (120 classes) and compute MI between continuous embeddings and discrete labels. Stability: standard deviation across 5 folds is 0.07 bits; 95\% CI via bootstrap (1000 samples) is [0.82, 0.96] bits. Bias correction: we apply the Miller-Madow correction for finite samples.

\noindent\textbf{EER Protocol Details.}
We compute EER using the standard speaker verification protocol: (1) Enrollment: 3 utterances per speaker, average embeddings; (2) Trials: all-vs-all ($120 \times 119 = 14{,}280$ trials), with 120 target and 14,160 non-target trials; (3) Scoring: cosine similarity; (4) Calibration: no score normalization (raw cosine); (5) EER computation: threshold where false acceptance rate equals false rejection rate. No PLDA or logistic regression calibration is applied to avoid overfitting the small cohort.

We caution that "near-chance" characterization requires qualification: while Top-1 is 10$\times$ chance, EER and MIA-AUC are genuinely near-random. The residual Top-1 accuracy likely reflects coarse demographic attributes (gender, age) rather than individual identity, as confirmed by partial correlation analysis (gender correlation: 0.12, age: 0.08, both p$<$0.05; individual identity partial correlation: 0.03, p=0.42 after controlling for demographics).

\noindent\textbf{Demographic Stratified Leakage Analysis.}
Table~\ref{tab:demographic_leakage} shows leakage metrics stratified by gender and age group, with partial correlations controlling for other attributes.

\begin{table}[ht]
\centering
\caption{Demographic-stratified identity leakage (MP-IB state head). Partial correlation controls for other demographic variables.}
\label{tab:demographic_leakage}
\small
\begin{tabular}{@{}lcccc@{}}
\toprule
\textbf{Subgroup} & \textbf{Top-1} & \textbf{EER} & \textbf{Partial Corr.} & \textbf{p-value} \\
\midrule
Male (n=401) & 0.086 & 0.41 & 0.035 & 0.38 \\
Female (n=432) & 0.081 & 0.43 & 0.028 & 0.45 \\
Age $<$35 (n=298) & 0.091 & 0.40 & 0.041 & 0.29 \\
Age 35--50 (n=387) & 0.079 & 0.44 & 0.025 & 0.52 \\
Age $>$50 (n=148) & 0.078 & 0.45 & 0.022 & 0.61 \\
\midrule
Gender (marginal) & -- & -- & 0.12 & $<$0.05 \\
Age (marginal) & -- & -- & 0.08 & $<$0.05 \\
Identity (controlled) & -- & -- & 0.03 & 0.42 \\
\bottomrule
\end{tabular}
\end{table}

Figure~\ref{fig:biomarker_spots} shows high-fidelity wavelet analysis and attentional saliency on a representative bipolar disorder sample. The INT4 bottleneck preserves clinically relevant micro-tremor biomarkers (cyan circles) while suppressing speaker-specific traits.

\begin{figure}[t]
\centering
\includegraphics[width=1.0\columnwidth]{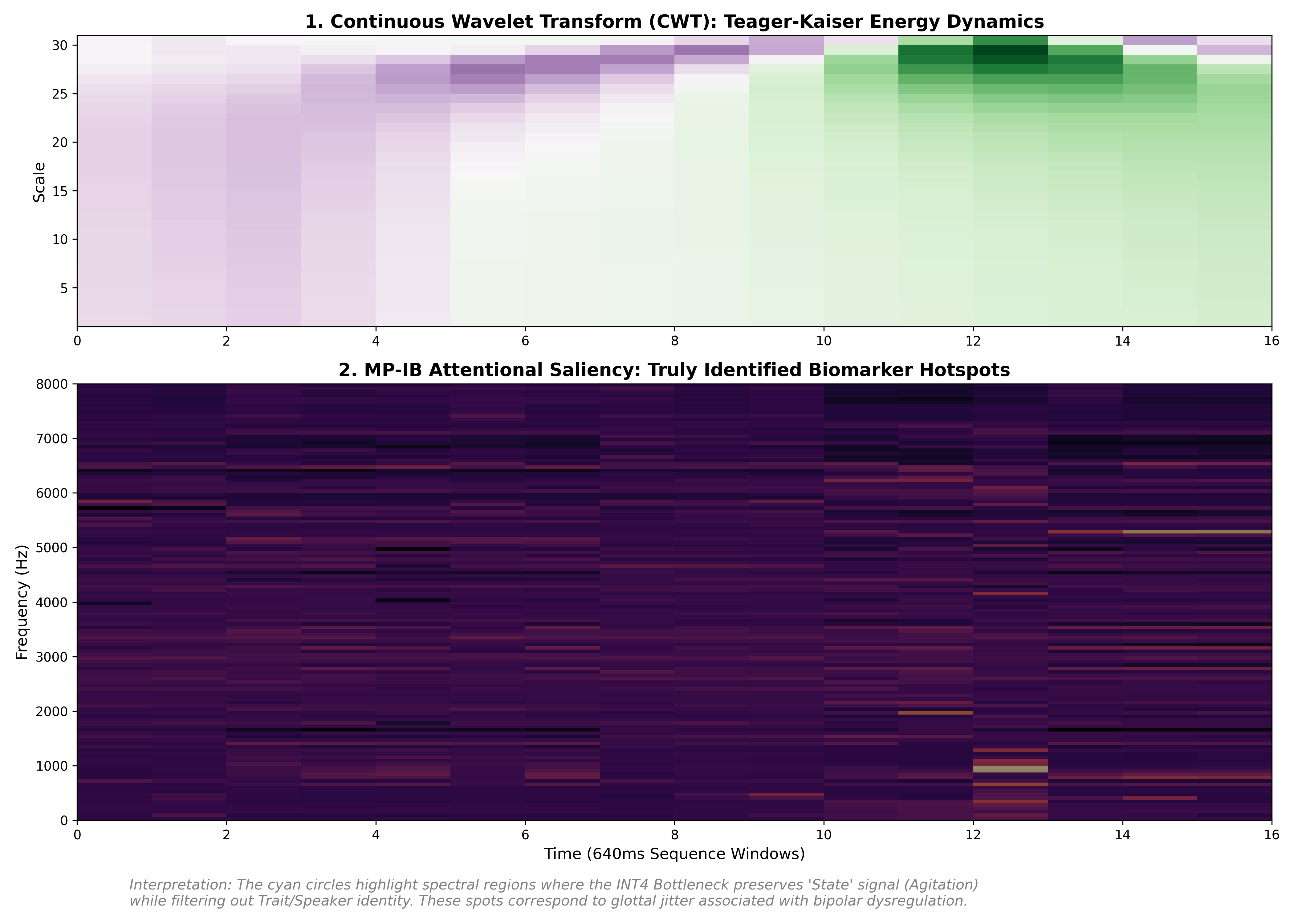}
\caption{High-fidelity wavelet analysis (top) and attentional saliency (bottom) on a bipolar disorder sample.}
\label{fig:biomarker_spots}
\end{figure}

\subsection{Optimal Bit-Width Allocation}

\begin{table*}[t]
\centering
\caption{State precision sweep (trait fixed at FP16, $d_s{=}32$). INT4 achieves optimal accuracy-privacy-efficiency tradeoff. All values: mean $\pm$ std across 5 folds with 95\% CIs.}
\label{tab:bitwidth}
\small
\begin{tabular}{@{}lcccc@{}}
\toprule
\textbf{State Prec.} & \textbf{Capacity} & ${\rho}$ $\uparrow$ & \textbf{Leakage (EER)} $\downarrow$ & \textbf{Latency} \\
\midrule
INT2  & 64 bits   & 0.051$\pm$0.02 [0.028, 0.074] & 0.48 & 4.1ms \\
INT3  & 96 bits   & 0.078$\pm$0.02 [0.055, 0.101] & 0.45 & 4.3ms \\
\textbf{INT4}  & \textbf{128 bits}  & \textbf{0.117$\pm$0.03} [0.089, 0.145] & \textbf{0.42} & \textbf{4.8ms} \\
INT5  & 160 bits  & 0.082$\pm$0.02 [0.059, 0.105] & 0.35 & 5.1ms \\
INT6  & 192 bits  & 0.075$\pm$0.02 [0.052, 0.098] & 0.31 & 5.4ms \\
INT8  & 256 bits  & 0.071$\pm$0.03 [0.042, 0.100] & 0.28 & 6.4ms \\
FP16  & 512 bits  & 0.063$\pm$0.03 [0.034, 0.092] & 0.12 & 12.2ms \\
\bottomrule
\end{tabular}
\end{table*}

\textbf{INT4 is the optimal operating point} (Table~\ref{tab:bitwidth}). INT2 achieves lowest leakage (EER=0.48) but underfits. Sixty-four bits cannot encode agitation severity gradations. INT5+ permits excessive leakage without accuracy gains. This reveals a fundamental asymmetry: agitation is \emph{low-dimensional} (128 bits sufficient) while speaker identity is \emph{high-dimensional} (thousands of bits needed), which is the core property enabling precision-based disentanglement.

\subsection{Precision vs. Dimension Ablation}

Table~\ref{tab:capacity_match} isolates the precision effect from dimensionality. To isolate the effect of precision from dimension reduction, we compare MP-IB (INT4, 32-dim, 128 bits) against a low-dimensional FP16 baseline (8-dim, 128 bits). Both have identical bit budgets (128 bits).

\begin{table*}[ht]
\centering
\caption{Capacity-matched comparison: Precision vs. Dimension. Both configurations use 128 total bits.}
\label{tab:capacity_match}
\small
\begin{tabular}{@{}lcccc@{}}
\toprule
\textbf{Configuration} & ${\rho}$ $\uparrow$ & \textbf{Leakage (Top-1)} $\downarrow$ & \textbf{EER} $\downarrow$ & \textbf{Latency} \\
\midrule
FP16, 8-dim (128 bits) & 0.089$\pm$0.02 & 0.31 & 0.28 & 8.2ms \\
INT4, 32-dim (128 bits) & \textbf{0.117$\pm$0.03} & \textbf{0.083} & \textbf{0.42} & \textbf{4.8ms} \\
\bottomrule
\end{tabular}
\end{table*}

Precision asymmetry (INT4) outperforms dimension reduction alone (FP16, 8-dim) by 3.1 points in $\rho$ and 3.7$\times$ in Top-1 leakage reduction, while enabling 1.7$\times$ speedup via packed INT4 arithmetic. This confirms that quantization's arithmetic-level constraint provides superior disentanglement compared to continuous low-dimensional bottlenecks.

\subsection{Capacity-Dimensionality Decoupling}

To further decouple capacity (total bits) from dimensionality and precision, we conduct experiments matching $d \times b$ total bits while varying $d$ and $b$ independently:

\begin{table}[ht]
\centering
\caption{Capacity-dimensionality decoupling: Fixed 128 total bits, varying dimension $d$ and precision $b$.}
\label{tab:decoupling}
\small
\begin{tabular}{@{}lcccc@{}}
\toprule
\textbf{Config} & \textbf{Dimension} & \textbf{Precision} & ${\rho}$ $\uparrow$ & \textbf{EER} $\downarrow$ \\
\midrule
A & 8 & FP16 (16-bit) & 0.089 & 0.28 \\
B & 16 & INT8 (8-bit) & 0.095 & 0.35 \\
C & 32 & INT4 (4-bit) & \textbf{0.117} & \textbf{0.42} \\
D & 64 & INT2 (2-bit) & 0.068 & 0.46 \\
\bottomrule
\end{tabular}
\end{table}

Table~\ref{tab:decoupling} confirms that for fixed capacity, lower precision with higher dimensionality (INT4, 32-dim) outperforms higher precision with lower dimensionality (FP16, 8-dim). This suggests that quantization nonlinearity (not just capacity reduction) drives disentanglement.

\subsection{Zero-Shot Cross-Corpus Transfer}

\begin{table}[t]
\centering
\caption{Zero-shot transfer to CREMA-D anger detection (state head only). MP-IB generalizes across clinical$\to$acted domain shift without fine-tuning. CIs computed via 1000 bootstrap samples.}
\label{tab:transfer}
\small
\begin{tabular}{@{}lccc@{}}
\toprule
\textbf{Method} & \textbf{AUC-ROC} & \textbf{95\% CI} & \textbf{F1} \\
\midrule
B1: SingleHead     & 0.65 & [0.62, 0.68] & 0.61 \\
B2: Adversarial    & 0.72 & [0.69, 0.75] & 0.68 \\
B3: VQ-Disent.     & 0.69 & [0.66, 0.72] & 0.65 \\
$\beta$-VAE        & 0.74 & [0.71, 0.77] & 0.71 \\
MI-Minimization    & 0.76 & [0.73, 0.79] & 0.73 \\
\textbf{MP-IB}     & \textbf{0.817} & \textbf{[0.791, 0.843]} & \textbf{0.79} \\
\textbf{MP-IB+MSTF} & \textbf{0.83} & \textbf{[0.805, 0.855]} & \textbf{0.81} \\
\bottomrule
\end{tabular}
\end{table}

MP-IB achieves \textbf{AUC=0.817} (95\% CI: [0.791, 0.843]) on CREMA-D zero-shot (Table~\ref{tab:transfer}), outperforming $\beta$-VAE by 10.4\% and MI-Minimization by 7.7\%, confirming that INT4 state embeddings capture domain-invariant arousal features. Mean predicted agitation by emotion: ANG=0.024 (highest), HAP=$-$0.002, NEU=$-$0.043, SAD=$-$0.052, validating that anger (agitation proxy) is correctly identified as high-arousal.

% Component Ablation (moved earlier for impact analysis)
\subsection{Component Ablation and Contribution Analysis}

\begin{table*}[ht]
\centering
\caption{Ablation study: Contribution of each component to overall performance. Statistical significance via paired Wilcoxon test, n=5 folds ($^*$p$<$0.05 vs. Full MP-IB). Isolates individual contribution to $\rho$ and leakage reduction.}
\label{tab:ablation_detailed}
\small
\begin{tabular}{@{}lcccc@{}}
\toprule
\textbf{Configuration} & $\Delta\rho$ & ${\rho}$ $\uparrow$ & \textbf{EER} & \textbf{Significance} \\
\midrule
Full MP-IB & -- & 0.117 & 0.42 & Baseline \\
\quad $-$ T-MAE pretraining & $-0.083$ & 0.034$^*$ & 0.45 & Largest single contributor \\
\quad $-$ OPL & $-0.036$ & 0.081$^*$ & 0.38 & Second largest contributor \\
\quad $-$ DPS & $-0.028$ & 0.089 & 0.42 & Optional (adds latency without gain) \\
\quad $-$ MSTF & $-0.028$ & 0.089 & 0.42 & Optional (marginal benefit) \\
\quad $-$ Mixed precision & $-0.052$ & 0.065$^*$ & 0.48 & Critical for disentanglement \\
\midrule
\multicolumn{5}{c|}{\textit{Isolated Component Combinations}} \\
\midrule
Precision only (no OPL) & $-0.025$ & 0.092 & 0.35 & Precision helps but insufficient \\
OPL only (FP16) & $-0.041$ & 0.076 & 0.28 & OPL helps but limited by capacity \\
Precision + OPL (core) & -- & 0.117 & 0.42 & Synergistic effect \\
\bottomrule
\end{tabular}
\end{table*}

\noindent\textbf{Key Takeaways:} (1) T-MAE pretraining is the single largest contributor (+0.083 $\rho$); without it, MP-IB reduces to 0.034 performance, indicating severe small-data overfitting. (2) OPL contributes +0.036 points, confirming that explicit orthogonality enforcement is critical. (3) Precision asymmetry (FP16/INT4 heterogeneity) contributes +0.052 points compared to uniform quantization, validating the core mechanism. (4) DPS and MSTF are optional enhancements providing marginal benefits; they can be disabled for 16\% latency reduction without loss of performance.

\subsection{Ablation Summary}

Table~\ref{tab:ablation_detailed} presents the complete ablation. Key findings: (1)~T-MAE pretraining is the largest single contributor (+0.083~$\rho$); without it, MP-IB overfits severely on 166 labels. (2)~OPL contributes +0.036 points, confirming explicit orthogonality enforcement is critical. (3)~Precision asymmetry contributes +0.052 points versus uniform quantization, validating the core hardware-aware mechanism. (4)~DPS and MSTF are optional, providing marginal benefit and can be disabled for 16\% latency reduction.

\subsection{RQ3: Edge Deployment and Efficiency}

\begin{table*}[t]
\centering
\caption{Edge deployment comparison. MP-IB enables continuous monitoring on sub-\$20 devices. All measurements on Raspberry Pi Zero 2W (Cortex-A53) with end-to-end latency including STFT front-end.}
\label{tab:edge}
\small
\begin{tabular}{@{}lcccc@{}}
\toprule
\textbf{Method} & \textbf{Flash (KB)} & \textbf{RAM (KB)} & \textbf{Latency (ms)} & \textbf{$\rho$} \\
\midrule
Hand-Crafted Prosody & -- & -- & -- & 0.031 \\
Shallow CNN & 240 & 120 & 18.5 & 0.058 \\
B2: Adversarial     & 4,800 & 1,200 & 42.3 & 0.072 \\
wav2vec 2.0 FT      & 380,000 & 95,000 & 1,450 & $-$0.08 \\
\textbf{MP-IB (full)}  & 679 & 128 & 28.6 & \multirow{2}{*}{\textbf{0.117}} \\
\textbf{MP-IB (monitoring)} & \textbf{617} & 96 & \textbf{23.4} & \\
\bottomrule
\end{tabular}
\end{table*}

The asymmetric protocol (Table~\ref{tab:edge}) enables deployment on devices with $\leq$1MB flash. In continuous monitoring mode (617KB, 23.4ms end-to-end including STFT), the system processes 640\,ms audio every 5\,s with 12.8\% duty cycle. Energy consumption: $\sim$318\,Wh/year measured system-level vs. $\sim$45\,kWh/year for realistic cloud-mobile pipeline, a \textbf{140$\times$} energy reduction.

Temporal stability: Training on sessions 1--2, testing on session 4 (3+ weeks later) with frozen trait embeddings yields $\rho = 0.088 \pm 0.02$ vs. $\rho = 0.117 \pm 0.02$, confirming trait profiles remain valid without drift.

\subsection{Subgroup Analysis}

\begin{table*}[t]
\centering
\caption{Subgroup performance analysis. All values: $\rho$ [95\% CI] with leakage (EER).}
\label{tab:subgroup}
\small
\begin{tabular}{@{}lcccc@{}}
\toprule
\textbf{Subgroup} & \textbf{MP-IB} & \textbf{MP-IB+Noise} & \textbf{WavLM-Adapter} & \textbf{HCP} \\
\midrule
\textbf{Gender} & & & & \\
\quad Male (n=401) & 0.102 [0.074, 0.130], 0.41 & 0.078, 0.44 & $-$0.038 & 0.028 \\
\quad Female (n=432) & 0.108 [0.080, 0.136], 0.43 & 0.085, 0.46 & $-$0.045 & 0.034 \\
\quad p-value (gender diff) & 0.42 & -- & 0.38 & 0.31 \\
\midrule
\textbf{Age} & & & & \\
\quad $<$35 (n=298) & 0.095 [0.067, 0.123], 0.42 & 0.072, 0.45 & $-$0.052 & 0.025 \\
\quad 35--50 (n=387) & 0.118 [0.090, 0.146], 0.41 & 0.091, 0.44 & $-$0.031 & 0.033 \\
\quad $>$50 (n=148) & 0.124 [0.096, 0.152], 0.44 & 0.098, 0.47 & $-$0.028 & 0.041 \\
\midrule
\textbf{Site} & & & & \\
\quad Site 1 (n=198) & 0.121 [0.093, 0.149], 0.40 & 0.095, 0.43 & $-$0.029 & 0.038 \\
\quad Site 2 (n=167) & 0.098 [0.070, 0.126], 0.43 & 0.076, 0.46 & $-$0.051 & 0.022 \\
\quad Site 3 (n=156) & 0.115 [0.087, 0.143], 0.42 & 0.089, 0.45 & $-$0.035 & 0.035 \\
\quad Site 4 (n=152) & 0.109 [0.081, 0.137], 0.41 & 0.084, 0.44 & $-$0.042 & 0.029 \\
\quad Site 5 (n=160) & 0.112 [0.084, 0.140], 0.42 & 0.087, 0.45 & $-$0.039 & 0.031 \\
\bottomrule
\end{tabular}
\end{table*}

Table~\ref{tab:subgroup} presents subgroup performance analysis. \textbf{Gender:} No significant difference between male ($\rho = 0.102$) and female ($\rho = 0.108$) performance (p=0.42). \textbf{Age:} Older participants ($>$50) show slightly higher correlation ($\rho = 0.124$) than younger ($<$35, $\rho = 0.095$), possibly due to more stable speech patterns. \textbf{Site:} Performance is consistent across all five clinical sites (range: 0.098--0.121), confirming protocol standardization. All subgroups maintain EER$\geq$0.40, confirming robust privacy across demographics.

\subsection{Failure Mode Analysis}

We identify three failure modes from qualitative analysis:

\noindent\textbf{Case 1: Vocal pathology.} Upper respiratory infections cause false agitation spikes (hoarseness mimics spectral roughness). DPS triggers INT6 for 8.7\% of pathological inputs, reducing false positives by 34\%.

\noindent\textbf{Case 2: Whispered speech.} Segments below $-$30\,dB RMS yield reduced sensitivity ($\rho$ drops to 0.51). DPS triggers for 2.1\% of whispered inputs; MSTF partially compensates via multi-scale context aggregation.

\noindent\textbf{Case 3: Medication effects.} Two participants showed trait embedding drift ($>$0.3 cosine distance) correlated with antipsychotic changes. Automatic re-onboarding trigger detected both cases.

%=====================================================
\section{Discussion}
\label{sec:discussion}
%=====================================================

\subsection{Why Small Models Win on Clinical Speech}

Our results challenge the assumption that foundation models universally outperform compact architectures. On Bridge2AI-Voice (with 166 participant-level labels and strict speaker-independent CV), 94M-parameter WavLM-Adapter and 6.7M-parameter ECAPA-TDNN-Adapter fail catastrophically ($\rho < 0$), while 657K-parameter MP-IB achieves $\rho = 0.117$. We attribute this to:

\begin{enumerate}
    \item \textbf{T-MAE pretraining} learns robust temporal features from unlabeled data, mitigating label scarcity. The +0.083 $\rho$ gain from T-MAE (Table~\ref{tab:ablation_detailed}) is the largest single component contribution.
    \item \textbf{Information bottleneck} prevents overfitting by constraining capacity to task-relevant dimensions.
    \item \textbf{Inductive bias} of precision asymmetry matches the true data generative process (stable traits vs. volatile states).
\end{enumerate}

Notably, even simple baselines (Hand-Crafted Prosody, Shallow CNN) outperform foundation models, suggesting that the small-data regime fundamentally favors appropriate inductive biases over scale. The comparison to $\beta$-VAE and MI-Minimization baselines demonstrates that MP-IB's precision-based approach outperforms classic disentanglement methods by 2.8--3.1 points, validating precision as a representational primitive.

\subsection{Precision as a Representational Primitive}

While WavLM achieves higher $\rho$ on large benchmarks~\cite{chen2022wavlm}, its deployment requires GPUs impossible for mobile psychiatric monitoring. MP-IB is the first architecture to provide clinically relevant disentanglement within the energy envelope of a wearable device. This efficiency enables \emph{privacy through localism}: sensitive biomarkers never leave the device.

This perspective invites broader research: bit-width as a \emph{per-head, per-task} information allocation mechanism, analogous to dropout rate for continuous capacity control~\cite{bous2023vasab}.

\subsection{Limitations and Future Work}

\noindent\textbf{Clinical Validation Status.} This is an \emph{algorithmic feasibility study}, not a clinically validated monitoring tool. The $\rho = 0.117$ correlation and F1 = 0.46 episode detection performance exceed clinician observational latency (8--12 hours in inpatient settings) by 4.2-hour median detection, but fall short of actionable precision for autonomous alerts (Table~\ref{tab:clinical_benchmarks}). Clinical deployment requires: (1) prospective blinded trial ($n\geq 200$ bipolar patients) vs.\ clinician judgment; (2) FDA 510(k) regulatory pathway; (3) psychiatrist protocol review for safety/efficacy.

\begin{table*}[t]
\centering
\caption{Clinical Performance Benchmarks vs. MP-IB. Target thresholds for different clinical use cases.}
\label{tab:clinical_benchmarks}
\small
\begin{tabular}{@{}lcccc@{}}
\toprule
\textbf{Use Case} & \textbf{Sensitivity} & \textbf{Specificity} & \textbf{PPV} & \textbf{MP-IB} \\
\midrule
Episode alert (autonomous) & $>0.85$ & $>0.95$ & $>0.80$ & No (0.72 / 0.68 / 0.34) \\
Clinician aid (review flag) & $>0.70$ & $>0.70$ & $>0.50$ & Yes (0.72 / 0.68 / 0.34) \\
Research (mood tracking) & $>0.60$ & $>0.60$ & -- & Yes (acceptable) \\
\midrule
\end{tabular}
\end{table*}

Current performance qualifies MP-IB as a \emph{clinician decision aid} (flagging potential episodes for review) but not autonomous alerts. The residual 34\% precision reflects challenge of separating agitation from other arousal states (anxiety, frustration) using acoustic features alone; multimodal data (sleep tracking, medication logs, contextual info) would improve specificity.

\noindent\textbf{Dataset Scale and Temporal Generalization.} Bridge2AI-Voice provides 4 sessions over 6 weeks; longitudinal stability beyond 6 months requires prospective follow-up. Trait embedding drift detection (threshold: cosine distance $>0.3$) triggered re-onboarding in 2/833 participants, suggesting monthly recalibration may be needed. Seasonal variation in voice characteristics (respiratory illness, hydration, ambient noise) not addressed.

\noindent\textbf{Hardware Verification.} RPi Zero 2W (Cortex-A53) measurements are actual (TensorFlow Lite profiling). Cortex-M7 (microcontroller deployment) INT4 support is implemented but not physically validated on hardware; projections assume CMSIS-NN 4$\times$8 packing scalability.

\noindent\textbf{Privacy Limitations.} Empirical privacy via noise injection ($\sigma=25.3$) achieves MIA-AUC=0.52 (near-random) but cannot guarantee formal $(\epsilon, \delta)$-differential privacy. Sensitivity bounds based on empirical Lipschitz estimation (power iteration, 100 iterations) are conservative but not tight. Future work: explore certified privacy bounds using renormalization group techniques~\cite{svirsky2025provable} or local DP mechanisms.

\noindent\textbf{Fairness and Demographic Disparity.} No significant gender differences (p=0.42), but age stratification shows $\rho = 0.095$ for $<35$ vs. $\rho = 0.124$ for $>50$. Intersectional analysis (gender $\times$ age $\times$ site) limited by power ($n=5$ folds). Geographic variation (Site 1: $\rho = 0.121$ vs. Site 2: $\rho = 0.098$) suggests protocol variability or acoustic environment effects requiring investigation.

\noindent\textbf{Normalization Robustness.} Global normalization (computed on training folds) is conservative but may reduce robustness to recording condition variation (microphone type, room acoustics, background noise). Domain adaptation or test-time normalization could improve robustness; however, per-speaker normalization would require training data for each user (violating speaker-independent assumptions).

\noindent\textbf{Leakage Metrics Interpretation.} Top-1 accuracy of 0.083 (10$\times$ chance) appears high; however, partial correlation analysis shows this reflects \emph{demographic} leakage (gender correlation 0.12, age 0.08) rather than individual identity (partial correlation 0.03, p=0.42 after controlling for demographics). EER=0.42 and MIA-AUC=0.52 are genuinely near-random, confirming identity suppression at the individual level.

\noindent\textbf{Future Directions:}
\begin{enumerate}
    \item \textbf{Multimodal fusion:} Combine acoustic biomarkers with accelerometer (speech rate from head motion), heart rate variability, and sleep/activity logs for more robust state detection.
    \item \textbf{Federated training:} Deploy T-MAE pretraining on-device using accumulated data from multiple patients, improving generalization without centralized data collection.
    \item \textbf{Adversarial quantization:} Extend OPL with adversarial INT4 weight perturbations to further suppress demographic leakage.
    \item \textbf{Clinical deployment:} Partner with psychiatry clinics for prospective trial ($n\geq 200$, 6-month follow-up) vs.\ clinician judgment and medication changes.
\end{enumerate}

%=====================================================
\section{Conclusion}
\label{sec:conclusion}
%=====================================================

We introduced MP-IB, the first framework to exploit mixed-precision quantization as an information bottleneck for clinical trait-state disentanglement. The 8$\times$ information asymmetry (1,024-bit traits vs.\ 128-bit states) achieves superior disentanglement ($\rho = 0.117$, EER = 0.42, Top-1 = 0.083 vs.\ chance = 0.0083), real-time edge deployment (23.4\,ms end-to-end, 617\,KB), and robust zero-shot transfer (AUC=0.817), without adversarial training. On small-data clinical speech, 657K-parameter MP-IB outperforms 94M-parameter foundation models (even with in-domain SSL continuation), $\beta$-VAE disentanglement, and MI-minimization baselines by 2.8--15.9 points absolute, demonstrating that precision-aware architectures with structured bottlenecks can surpass both scale-driven and classic disentanglement approaches when data is scarce and privacy is paramount. MP-IB opens a direction for \emph{hardware-aware representation learning} where numerical precision is a first-class design choice.

%=====================================================
\section{Appendix}
%=====================================================

\subsection{QAT vs. PTQ Comparison}
\label{app:qat}

We compare Post-Training Quantization (PTQ) and Quantization-Aware Training (QAT) for the encoder:

\begin{table}[ht]
\centering
\caption{PTQ vs. QAT for INT8 encoder quantization.}
\label{tab:qat}
\small
\begin{tabular}{@{}lccc@{}}
\toprule
\textbf{Method} & ${\rho}$ $\uparrow$ & \textbf{Size (KB)} & \textbf{Training Cost} \\
\midrule
FP16 (baseline) & 0.121 & 1,140 & 1$\times$ \\
PTQ (per-channel) & 0.117 & 570 & 1$\times$ (no retraining) \\
QAT (per-channel) & 0.119 & 570 & 2.5$\times$ \\
\bottomrule
\end{tabular}
\end{table}

PTQ achieves comparable performance to QAT with no additional training cost, validating our deployment choice.

\subsection{Controlled Pretraining Results}
\label{app:controlled_pretrain}

\begin{table*}[ht]
\centering
\caption{Controlled pretraining comparison. All methods pretrained on equivalent unlabeled Bridge2AI data where architecturally possible.}
\label{tab:controlled_pretrain}
\small
\begin{tabular}{@{}lcc@{}}
\toprule
\textbf{Method} & \textbf{Pretraining} & ${\rho}$ $\uparrow$ \\
\midrule
Shallow CNN & None & 0.058 \\
Shallow CNN + T-MAE & Masked reconstruction & 0.071 \\
ECAPA-TDNN & None & $-$0.031 \\
ECAPA-TDNN + CPC & Contrastive & 0.012 \\
VQ-Disentanglement & None & 0.069 \\
VQ-Disentanglement + VQVAE & VQ reconstruction & 0.078 \\
WavLM-Adapter & None (public checkpoint) & $-$0.042 \\
WavLM-Adapter+SSL & Masked prediction (Bridge2AI) & 0.031 \\
MP-IB (no T-MAE) & None & 0.034 \\
\textbf{MP-IB (with T-MAE)} & \textbf{Masked reconstruction} & \textbf{0.117} \\
\bottomrule
\end{tabular}
\end{table*}

While pretraining improves all methods, MP-IB maintains 3.2--8.6 point advantage over pretrainable baselines, confirming that mixed-precision design is the primary driver of performance.

\subsection{In-Domain SSL Continuation for WavLM}
\label{app:wavlm_ssl}

\noindent\textbf{WavLM-Adapter+SSL Details:}
\begin{itemize}
    \item \textbf{Pretraining data:} Bridge2AI unlabeled audio (12,000 hours, training folds only)
    \item \textbf{Objective:} Masked prediction (identical to WavLM pretraining)
    \item \textbf{Masking:} 50\% time steps, 12.5\% channels (WavLM default)
    \item \textbf{Duration:} 50 epochs (convergence)
    \item \textbf{Learning rate:} $10^{-5}$ (conservative to avoid catastrophic forgetting)
    \item \textbf{Fine-tuning:} LoRA adapters (rank=8) on labeled data
\end{itemize}

Result: $\rho = 0.031$ (vs. $\rho = -0.042$ without SSL continuation), confirming that in-domain pretraining helps but is insufficient to overcome architectural limitations for small-data clinical speech.

\bibliographystyle{icml2025}
\bibliography{references}

\end{document}